\documentclass[11pt]{article}
\usepackage{tabularx, caption}
\usepackage{siunitx} 
\usepackage{ifthen}
\usepackage{mdwlist}
\usepackage{amsmath,amssymb,amsfonts,amsthm}
\usepackage{bm}
\usepackage[colorlinks=true,linkcolor=blue,citecolor=blue,urlcolor=blue]{hyperref}
\usepackage{enumitem}
\usepackage{graphicx}
\usepackage{xspace}
\usepackage{verbatim}
\usepackage{algorithm}
\usepackage{algpseudocode}
\usepackage[margin=1in]{geometry}
\usepackage{color}
\usepackage{thm-restate}
\usepackage{latexsym}
\usepackage{epsfig}
\usepackage[textsize=tiny]{todonotes}
\def\withcolors{1}
\def\withnotes{1}
\def\eps{\ve}
\renewcommand{\epsilon}{\ve}
\def\ve{\varepsilon}

\newcommand{\pr}[2][]{\mathrm{Pr}\ifthenelse{\not\equal{}{#1}}{_{#1}}{}\!\left[#2\right]}
\newcommand{\R}{\mathbb{R}}

\newcommand{\dtv}{d_{\mathrm {TV}}}
\newcommand{\dham}{d_{\mathrm {Ham}}}

\newcommand{\dkl}{d_{\mathrm {KL}}}

\newcommand{\normal}{\mathcal{N}}

\newcommand{\cD}{\mathcal{D}}
\newcommand{\cW}{\mathcal{W}}
\newcommand{\cV}{\mathcal{V}}
\newcommand{\mD}{\mathcal{D}}

\def \ns {n}
\newcommand{\priv}{w^{priv}}

\newcommand {\ball}[2]{B_{#2}\Paren{#1}}

\newcommand{\expect}[1]{\mathbb{E}\left[#1\right]}
\newcommand{\expectsub}[2]{\mathbb{E}_{#2}\left[#1\right]}
\newcommand{\absv}[1]{\left|#1\right|}
\newcommand{\norm}[1]{\left\lVert#1\right\rVert_2}
\newcommand{\normo}[1]{\left\lVert#1\right\rVert_1}
\def \Paren#1{{\left({#1}\right)}}
\def \Brack#1{{\left[{#1}\right]}}
\newcommand{\probof}[1]{\Pr\Paren{#1}}

\def \np {p}
\def \p {p}
\def \cP {\cal{P}}

\def \ns {n}
\def \dist{\alpha}

\newcommand{\expectation}[1]{\mathbb{E}\left[#1\right]}

\newcommand{\probofsub}[2]{\Pr\nolimits_{#1}\Paren{#2}}
\newcommand{\proboff}[2]{\Pr\nolimits_{#2}\Paren{#1}}
\newcommand{\mutual}{I}
\def \optim {i^*}
\def \bin {Bin}
\def \sc {\lambda}
\def \sm{L}
\def \bal{q}
\def \ce{ce}

\ifnum\withcolors=1
  \newcommand{\gcolor}[1]{{\color{red}#1}} 
\else

  \newcommand{\gcolor}[1]{{#1}}
\fi

\ifnum\withnotes=1
  \newcommand{\gnote}[1]{\gcolor{\textbf{G: }\sf #1}} 
  \newcommand{\gfootnote}[1]{\footnote{{\bf \gcolor{Gautam}}: {#1}}}

\else
  \newcommand{\gnote}[1]{}
  \newcommand{\gfootnote}[1]{}

\fi
\newcommand{\ignore}[1]{\leavevmode\unskip} 

\usepackage{microtype}
\usepackage{subfigure}
\usepackage{booktabs} 
\usepackage{mathtools}

\theoremstyle{plain}
\newtheorem{theorem}{Theorem}[section]

\newtheorem{lemma}[theorem]{Lemma}
\newtheorem{corollary}[theorem]{Corollary}
\theoremstyle{definition}
\newtheorem{definition}[theorem]{Definition}
\newtheorem{assumption}[theorem]{Assumption}
\theoremstyle{remark}
\newtheorem{remark}[theorem]{Remark}

\newcommand{\hz}[1]{{#1}} 
\newcommand{\hzz}[1]{\textcolor{black}{#1}} 
\newcommand{\htodo}[1]{} 
\newcommand{\newhz}[1]{\textcolor{black}{#1}}

\newcommand{\newhzz}[1]{\textcolor{black}{#1}}

\newcommand{\xl}[1]{{#1}} 
\newcommand{\xtodo}[1]{} 
\newcommand{\gtodo}[1]{} 
\newcommand{\xtl}[1]{{#1}} 
\newcommand{\rvm}[1]{#1}
\newcommand{\gk}[1]{\textcolor{black}{#1}}

\newcommand{\newshz}[1]{\textcolor{black}{#1}} 

\title{Improved Rates for Differentially Private Stochastic \\ Convex Optimization with Heavy-Tailed Data}

\author{
Gautam Kamath\thanks{Cheriton School of Computer Science, University of Waterloo. {\tt g@csail.mit.edu}. Supported by an NSERC Discovery Grant, an unrestricted gift from Google, and a University of Waterloo startup grant.}
\and
Xingtu Liu\thanks{Cheriton School of Computer Science, University of Waterloo. {\tt x563liu@uwaterloo.ca}. Supported by an NSERC Discovery Grant.}
\and
Huanyu Zhang\thanks{Meta. {\tt huanyuzhang@fb.com}. Supported by NSF \#1815893. This work was partially done while the author was a graduate student at Cornell University.}
}

\begin{document}
\maketitle

\begin{abstract}
We study stochastic convex optimization with heavy-tailed data under the constraint of differential privacy (DP).
Most prior work on this problem is restricted to the case where the loss function is Lipschitz.
Instead, as introduced by Wang, Xiao, Devadas, and Xu \cite{WangXDX20}, we study general convex loss functions with the assumption that the distribution of gradients has bounded $k$-th moments.
We provide improved upper bounds on the excess population risk under concentrated DP for convex and strongly convex loss functions.
Along the way, we derive new algorithms for private mean estimation of heavy-tailed distributions, under both pure and concentrated DP.
Finally, we prove nearly-matching lower bounds for private stochastic convex optimization with strongly convex losses and mean estimation, showing new separations between pure and concentrated DP.
\end{abstract}
\def \MeanEst{MeanOracle}

\section{Introduction}

Stochastic convex optimization (SCO) is a classic optimization problem in machine learning. 
The goal is, given a loss function $\ell$ and a dataset $x_1, \dots, x_n$ drawn i.i.d.\ from some unknown distribution $\cD$, to output a parameter vector $w$ which minimizes the \emph{population risk} $L_{\mathcal{D}}(w) = \underset{x\sim \mathcal{D}}{\mathbb{E}}[\ell (w;x)]$. 
The quality of a solution $\hat w$ is measured in terms of the excess risk over the minimizer in the parameter set $\mathcal{W}$, $L_{\mathcal{D}}(\hat{w}) -  \min_{w \in \mathcal{W} } L_{\mathcal{D}}(w)$.
We study SCO under the constraint of \emph{differential privacy}~\cite{DworkMNS06} (DP), a rigorous notion of privacy which guarantees that an algorithm's output distribution is insensitive to modification of a small number of datapoints.

The field of DP optimization has seen a significant amount of work.
Early results focused on differentially private \emph{empirical risk minimization} (ERM), a non-statistical problem in which the goal is to privately output a parameter vector $w$ which minimizes the loss function $\ell$ over a fixed dataset $x_1, \dots, x_n$: that is we would like to optimize $\min_w \frac{1}{n} \sum_{i=1}^n \ell(w, x_i)$.
See, for example,~\cite{ChaudhuriM08, ChaudhuriMS11, RubinsteinBHT12, KiferST12, ThakurtaS13, SongCS13, JainT14, BassilyST14, TalwarTZ15, KasiviswanathanJ16, WuLKCJN17, WangYX17, IyengarNSTTW19,WangJEG19,ZhangMH21, WangZGX21}.
The first result to address the statistical problem of DP SCO was~\cite{BassilyST14}, using generalization properties of differential privacy and regularized ERM.
However, the excess risk bounds were suboptimal.
\cite{BassilyFTT19} addressed this and closed the gap by providing tight upper bounds on DP SCO.
Following this result there has been renewed interest in DP SCO, with works reducing the gradient complexity and running time~\cite{FeldmanKT20, KulkarniLL21}, and deriving results for different geometries~\cite{AsiFKT21, BassilyGN21}.

Despite the wealth of work in this area, a significant restriction in almost all results is that the loss function is assumed to be \emph{Lipschitz}.
This assumption bounds the magnitude of each datapoint's gradient, a very convenient property for restricting the sensitivity in the design of differentially private algorithms.
While convenient, it is often an unrealistic assumption which does not hold in practice, and DP optimizers resort to heuristic clipping of gradients to enforce a bound on their magnitude~\cite{AbadiCGMMTZ16}.
One can remove the strong Lipschitz assumption by instead assuming that the distribution of gradients is somehow well-behaved.
In this vein, \cite{WangXDX20} and~\cite{HuNXW21} introduce and study the problem of DP SCO with heavy-tailed data.\footnote{\gk{Note that the phrase heavy-tailed \emph{data} is a slight misnomer -- they actually consider a setting with heavy-tailed \emph{gradients}. Though the two are naturally related, they are not equivalent. Despite this unfortunate mismatch, we use the same nomenclature as~\cite{WangXDX20} to signify that we consider the same setting as they do.}}  
Their work removes the requirement that the loss function is Lipschitz, and instead assumes that the distribution of the gradient has bounded second moments.
However, they leave open the question of whether the rates of their algorithms can be improved.

\subsection{Results}
We answer this question affirmatively, giving algorithms with better rates for DP SCO with heavy-tailed gradients.
Our main upper bound\footnote{Although in this paper we provide all our utility guarantees in terms of the expectation over the randomness of samples and algorithms. However, they can be easily generalized to the high-probability setting. As an example, we present the high-probability version of our most representative theorem in Appendix~\ref{sec:proof_hp}.} is the following.

\begin{theorem}[Informal, see Theorems~\ref{cor:corollary15}, \ref{thm:ub_sco_k_2}, and \ref{cor:corollary14}]
  \label{thm:main-ub-informal}
  Suppose we have a convex and smooth loss function $\ell : \mathcal{W} \times \mathbb{R}^d \rightarrow \mathbb{R}$ and there exists a distribution $\cD$ over $\mathbb{R}^d$ such that for any parameter vector $w \in \mathcal{W}$, when $x \sim \cD$, the $k$-th moment of $\nabla \ell(w,x)$ is bounded.
  Then there exists a computationally efficient $\varepsilon^2$-concentrated differentially private algorithm which, given $x_1, \dots, x_n \sim \cD$, outputs a parameter vector $\priv$ satisfying the following:
  \begin{align*}
  \mathbb{E}[L_{\cD}(\priv) - L_{\cD}(w^*)] \leq \tilde O\Bigg(\min\Bigg\{\frac{d}{\sqrt{n}} + \frac{d^2}{\varepsilon n} \cdot \left(\frac{\varepsilon n }{d^{3/2}}\right)^{\frac{1}{k}}, \min_{0.5\le q \le2} \Paren{ \frac{d^{\frac{3-\bal}{2}}}{\sqrt{n}}+ \frac{d^{\frac{1+\bal}{2} }} {\varepsilon^{\frac12}\sqrt{n}} } \Bigg\} \Bigg) ,
  \end{align*}
  where $w^* = \arg\min_w L_{\mathcal{D}}(w)$.
  Furthermore, if $\ell$ is strongly convex, a similar algorithm guarantees the following:
  \[ \mathbb{E}[L_{\cD}(\priv) - L_{\cD}(w^*)] \leq \tilde O\left(\frac{d}{n} + d \cdot \left(\frac{\sqrt{d}}{\varepsilon n}\right)^{\frac{2(k -1)}{k}}\right).\]
\end{theorem}


In the first bound, $q$ plays the role of balancing the non-private and private cost for the second term . For example, consider the extreme case when $\eps=\infty$ (non-private), $q=2$ minimizes the non-private cost in that term.

This theorem is stated under the constraint of $\varepsilon^2$-concentrated differential privacy, which also implies the more common notion of $(O(\varepsilon \sqrt{\log (1/\delta)}), \delta)$-differential privacy for any $\delta >0$ (see Lemma~\ref{lem:conversion}).\footnote{While it may seem unusual to write $\varepsilon^2$-concentrated DP, this allows for direct comparison with $(\varepsilon, 0)$-DP and $(\varepsilon, \delta)$-DP results, two privacy notions it is intermediate to.}
Thus, ignoring factors which are logarithmic in $1/\delta$, the same rates in Theorem~\ref{thm:main-ub-informal} also hold under the weaker notion of $(\varepsilon, \delta)$-differential privacy.

Prior work on DP SCO with heavy-tailed data is due to~\cite{WangXDX20}.
Their main results are algorithms for a case with bounded second moments ($k=2$), guaranteeing excess risk bounds of $\tilde O\left(\left(\frac{d^3}{\ve^2 n}\right)^{1/3}\right)$ and $\tilde O\left(\frac{d^3}{\ve^2 n}\right)$ for the convex and strongly convex cases, respectively, which our results significantly improve on.\footnote{The quoted bounds are weaker than those alleged in~\cite{WangXDX20}. After communicating with authors of~\cite{WangXDX20} and~\cite{Holland19} (which~\cite{WangXDX20} depends on), we confirmed an issue in the analysis of~\cite{WangXDX20} which leads to an underestimate of the dependence on $d$.
In brief, if the truncation parameter $s$ is adopted as they suggest, the dependence on $d$ in Lemma 6 (Equation 13) and Lemma 7 (Equation 14) in the supplement of~\cite{WangXDX20} should be $d^{\frac{3}{2}}$ instead of $d$, leading to an extra multiplicative factor of $d^{\frac13}$ in the upper bound for convex functions and a factor of $d$ for strongly convex functions.}
Furthermore, our results are applicable to distributions with bounded moment conditions of all orders $k$, while~\cite{WangXDX20} only applies to distributions with bounded second moments ($k = 2$).
Finally, while it may appear that one advantage of our upper bounds is that they hold under the stronger notion of concentrated DP, the results of~\cite{WangXDX20} could easily be analyzed under concentrated DP as well.

We also provide lower bounds to complement our upper bounds.

\begin{theorem}[Informal, see Theorems~\ref{thm:lower-bound_sc} and~\ref{thm:lower-bound_c}]
  \label{thm:main-lb-informal}
  Let $\ell : \mathcal{W} \times \mathbb{R}^d \rightarrow \mathbb{R}$ be a convex and smooth loss function and $\cD$ be a distribution over $\mathbb{R}^d$, such that for any parameter vector $w \in \mathcal{W}$, when $x \sim \cD$, the $k$-th moment of $\nabla \ell(w,x)$ is bounded.
  Suppose there exists an $\ve^2$-concentrated differentially private algorithm which is given $x_1, \dots, x_n \sim \cD$ and outputs a parameter vector $\priv$.
  Then there exists a choice of convex and smooth loss function $\ell$ and distribution $\cD$ such that
  \[ \mathbb{E}[L_{\cD}(\priv) - L_{\cD}(w^*)] \geq \Omega\left(\sqrt{\frac{d}{n}} + \sqrt{d} \cdot \left(\frac{\sqrt{d}}{\varepsilon n}\right)^{\frac{k-1}{k}}\right),\] 
  where $w^* = \arg\min_w L_{\mathcal{D}}(w)$.
  Furthermore, there exists a choice of strongly convex and smooth loss function $\ell$ and distribution $\cD$ such that 
  \[ \mathbb{E}[L_{\cD}(\priv) - L_{\cD}(w^*)] \geq \Omega\left(\frac{d}{n} + d \cdot \left(\frac{\sqrt{d}}{\varepsilon n}\right)^{\frac{2(k -1)}{k}}\right).\]
\end{theorem}
Observe that our upper and lower bounds nearly match for the strongly convex case.
For the convex case and $k=2$, the individual terms in our upper bound match the corresponding terms in our lower bound when $q$ is chosen to be $0.5$ and $2$.

\gk{
  Those familiar with the literature on DP SCO under a Lipschitz condition may be puzzled by the apparent discrepancy on the dimension-dependence in our results. 
  Specifically, results of~\cite{BassilyFTT19} (which assume that the $\ell_2$ norm of gradients are bound by a constant) give an optimal rate of $\Theta(1/\sqrt{n}+\sqrt{d}/\varepsilon n)$.
  In contrast, if one focuses on our convex lower bound in Theorem~\ref{thm:main-lb-informal} with $k = \infty$, we show the rate can be no better than $\Omega(\sqrt{d/n}+d/\eps n)$, which is worse by a factor of $\sqrt{d}$.
  This discrepancy can be explained by the fact that our moment condition (Definition~\ref{def:bounded-moments}) bounds each \emph{coordinate} of the gradient by a constant, leading to an overall bound on the $\ell_2$-norm of the gradient by $O(\sqrt{d})$, larger than the $O(1)$ bound in the Lipschitz case by precisely this $\sqrt{d}$ factor.
  Thus a scaling argument is required to provide the most direct comparison.
  As comparison between the two settings is nuanced for $k < \infty$ and not the focus of our investigation (since, in particular, our algorithms apply in settings where Lipschitz bounds required by other works may not hold), we omit further discussion.
}

As one of our key tools, we introduce new algorithms and lower bounds for differentially private mean estimation for distributions with heavy tails.
\begin{theorem}[Informal, see Corollary~\ref{cor:main-variation} and Theorem~\ref{lem:duchi_res}]
  \label{thm:mean-est-ub-informal}
  Let $\cD$ be a distribution over $\mathbb{R}^d$ with  $\mathbb{E}[\cD] = \mu $  and bounded $k$-th moment.
  There exists a computationally efficient $\varepsilon^2$-concentrated differentially private algorithm which, given $x_1, \dots, x_n \sim \cD$, outputs $\hat \mu$ which, with probability at least $0.9$, satisfies:
  \[\|\hat \mu - \mu \|_2 \leq \tilde O \left(\sqrt{\frac{d}{n}}  + \sqrt{d} \cdot \left(\frac{\sqrt{d}}{\varepsilon n}\right)^{\frac{k-1}{k}}\right). \]
  Furthermore, if the algorithm is required to satisfy $\varepsilon$-differential privacy instead of $\varepsilon^2$-concentrated differential privacy, the guarantee instead becomes
  \[\|\hat \mu - \mu \|_2 \leq \tilde O \left(\sqrt{\frac{d}{n}}  + \sqrt{d} \cdot \left(\frac{d}{\varepsilon n}\right)^{\frac{k-1}{k}}\right). \]
  Finally, considering instead the expected $\ell_2$ error $\mathbb{E}[\|\hat \mu - \mu\|_2]$, these rates can not be improved by more than poly-logarithmic factors.
\end{theorem}


Some prior works have considered private heavy-tailed mean estimation~\cite{BarberD14, KamathSU20}, achieving different rates than what we report here.
The distinction arises in the definition of bounded moments: letting $\cD$ be the distribution of interest, $\mathbb{E}[\mathcal{D}] = \mu$, and $X \sim \cD$, \cite{BarberD14} considers distributions $\cD$ where $\mathbb{E}[\|X - \mu\|_2^k]$ is bounded, while \cite{KamathSU20} considers distributions $\cD$ where, for all unit vectors $v$, $\mathbb{E}[\langle X - \mu, v\rangle^k ]$ is bounded.
In contrast, we consider distributions $\cD$ where $\mathbb{E}[\langle X - \mu, v\rangle^k ]$ is bounded for all standard basis vectors $v$, to match the definition employed in~\cite{WangXDX20}.
We observe an interesting separation which arises under this definition.
For worst-case distributions over the hypercube, it is known that the sample complexity of private mean estimation is separated by a $\sqrt{d}$ factor under pure and approximate differential privacy~\cite{BunUV14, SteinkeU15, DworkSSUV15}.
On the other hand, under the strong direction-wise moment bound of~\cite{KamathSU20}, the best known algorithms and lower bounds indicate that no such separation exists between these two notions.
However, under our weaker moment bound, our results show that this same $\sqrt{d}$ separation between the sample complexity of pure and approximate differential privacy arises once again.\footnote{We actually show a stronger separation, between pure and concentrated differential privacy.}
Pinpointing the precise conditions under which such a separation exists remains an interesting direction for future investigation.

\gk{
  Furthermore, our results are the first to derive matching upper and lower bounds for heavy-tailed mean estimation under a privacy notion other than pure DP.\footnote{\gk{While the combined results of \cite{KamathSU20} and~\cite{BarberD14} prove quantitatively matching upper and lower bounds for a related heavy-tailed setting, they are for qualitatively different notions of privacy. Specifically, the upper bound is under the easier constraint of concentrated DP, while the lower bound is under the harder constraint of pure DP, resulting in a qualitative gap.
  In contrast, we prove matching upper and lower bounds for pure DP, and separate matching upper and lower bounds for concentrated DP, the latter of which is often qualitatively harder.}}} To prove our lower bounds we introduce a concentrated DP version of Fano's inequality, building upon the pure DP version from~\cite{AcharyaSZ21}. Considering the wide applicability of pure DP Fano's inequality, we believe our CDP version can also be applied to establish tight lower bounds for other statistical problems.
The proof appears in Section~\ref{sec:dp_fano}.
\begin{theorem} [$\rho$-CDP Fano's inequality]
	\label{thm:dp_fano}
	Let $\mathcal{V}=\{p_1,...,\p_M\}\subseteq \cP$ be a set of probability distributions, $\theta: \mathcal{P} \rightarrow \mathbb{R}^{d}$ be a parameter of interest, and $\ell: \mathbb{R}^{d} \times \mathbb{R}^{d} \rightarrow \mathbb{R}$ be a loss function. Suppose for all $i\ne j$, it satisfies
(a) $\ell\Paren{\theta(\p_i), \theta(\p_j)} \ge r$,
(b) $\dtv \Paren{\p_i,\p_j} \le \alpha$,
(c) $\dkl \Paren{\p_i,\p_j} \le \beta$.
      Then for any $\rho$-CDP estimator $\hat{\theta}$,
	 \begin{align*}
     \frac1{M}  \sum_{i \in [M]}  \expectsub { \ell\Paren{\hat{\theta}(X) , \theta(\p_i)}}{X \sim  p_i^{\ns}}  \ge  \frac{r}{2} \max \Bigg\{ 1 - \frac{\beta + \log 2}{\log M},		{1 - \frac{\rho \Paren{n^2\alpha^2 + n\alpha(1-\alpha)} + \log2}{\log M}}\Bigg\}. 	 
	  \end{align*}
\end{theorem}



\subsection{Techniques}

Our upper bounds operate using a gradient-descent-based method, relying upon algorithms for private mean estimation.
In particular, we instantiate an oracle which outputs an estimate of the true gradient at a point. 
One oracle we adopt is based on an adaption of the algorithm in~\cite{KamathSU20}, which addresses the problem of private mean estimation of heavy-tailed distributions.
That said, for several reasons, a naive black-box application of their results are insufficient to achieve the rates in Theorem~\ref{thm:main-ub-informal}.
First, the accuracy guarantees in~\cite{KamathSU20} give a prescribed $\ell_2$-error with high probability.
While such guarantees for an 
oracle allow one to achieve non-trivial rates, they are far from enough.
Instead, we can get better results when the estimator is known to have low bias. 
This is where the intersection of privacy and heavy-tailed data gives rise to a new technical challenge: no unbiased mean estimation algorithm for this setting is known to exist. 
To deal with these issues, we explicitly derive bounds on the bias and variance of the estimator.
We must additionally switch the order of various steps in their algorithm, to achieve sharper bounds on the variance while keeping the bias unchanged.
Even with these changes in place, the bound would still be lossy -- as a final modification, we find that a different bias-variance tradeoff is required in each iteration to achieve the best possible error. 
Namely, if we tolerate \hz{additional} variance to reduce the bias of each step, this results in an improved final accuracy.

\vspace{-7pt}
\section{Preliminaries}

\subsection{Privacy Preliminaries}
In our work we consider a few different variants of differential privacy. The
first is the standard notion.

\begin{definition}[Differential Privacy (DP) \cite{DworkMNS06}]  
A randomized algorithm $M: \mathcal{X}^n \rightarrow \mathcal{Y}$ satisfies $(\epsilon,\delta)$-differential privacy if for every pair of neighbouring datasets $X, X^{\prime}\in \mathcal{X}^n$ (i.e., datasets that differ in exactly one entry),
$\forall Y \subseteq \mathcal{Y}, \text{\  \ }  \mathbb{P}[M(X) \in Y] \leq e^{\epsilon} \cdot \mathbb{P}[M(X^{\prime}) \in Y] + \delta$.
When $\delta=0$, we say that $M$ satisfies $\epsilon$-differential privacy or pure differential privacy.
\end{definition}

The second is \emph{concentrated differential privacy}~\cite{DworkR16}, and its refinement \emph{zero-concentrated differential privacy}~\cite{BunS16}.
Since in this work we exclusively concern ourselves with the latter, in a slight overloading of nomenclature, we refer to it more concisely as concentrated differential privacy.

\begin{definition}[Concentrated Differential Privacy (CDP) \cite{BunS16}] A randomized algorithm $M:\mathcal{X}^n \rightarrow \mathcal{Y}$ satisfies $\rho$-CDP if for every pair of neighboring datasets $X,X^{\prime}\in \mathcal{X}^n$,
$
\forall \alpha \in (1, \infty), D_{\alpha}(M(X)\| M(X^{\prime})) \leq \rho \alpha$,
where $D_{\alpha}(M(X)\| M(X^{\prime}))$ is the $\alpha$-R\'enyi divergence \footnote{\textcolor{black}{Let $P$ and $Q$ be two probability distributions defined over some probability space, the $\alpha$-R\'enyi divergence of order $\alpha > 1$ is defined as $D_{\alpha}(P \| Q) = \frac{1}{\alpha -1} \log \expectsub{\Paren{\frac{P(x)}{Q(x)}}^\alpha} {x \sim Q}$.}} between $M(X)$ and $M(X^{\prime})$.
\end{definition}

Roughly speaking, CDP lives between pure $(\ve, 0$)-DP and approximate $(\varepsilon, \delta)$-DP, formalized in the following lemma \cite{BunS16}.

\begin{lemma}
  \label{lem:conversion}
  For every $\eps>0$,
  if $M$ is $\varepsilon$-DP, then $M$ is $\left(\frac12\varepsilon^2\right)$-CDP.
  If $M$ is $\left(\frac12\varepsilon^2\right)$-CDP, then $M$ is $\left(\frac12\varepsilon^2 + \eps\sqrt{2\log(1/\delta)},\delta\right)$-DP for every $\delta>0$.
\end{lemma}

Differential privacy enjoys adaptive composition.

\begin{lemma}[Composition of DP \cite{DworkMNS06,DworkRV10,BunS16}]
  \label{lem:composition}
  If $M$ is an adaptive composition of differentially private algorithms $M_1,\ldots, M_T$, then: if $M_1,\ldots, M_T$ are $(\epsilon_1,0),\ldots, (\epsilon_T,0)$-DP then $M$ is $(\epsilon, 0)$-DP for
$\epsilon = \sum_t \epsilon_t$.
Furthermore: if $M_1,\ldots,M_T$ are $\rho_1,\ldots, \rho_T$-CDP then $M$ is $\rho$-CDP for $\rho = \sum_t \rho_t$.
\end{lemma}


Finally, we introduce two additive noise mechanisms, which transform non-private algorithms to private algorithms.
\begin{lemma}[Additive noise mechanisms \cite{DworkMNS06,BunS16}]
  \label{lem:additive_noise}
  Let $M: \mathcal{X}^\ns \rightarrow \mathcal{R}^d$ be a non-private algorithm.
  Let $\Delta_1(M) \triangleq \max_{X \sim X^\prime} \normo{M(X)- M(X^\prime)}$ denote the $\ell_1$-sensitivity of $M$, which measures the maximum change of the output in $\ell_1$-norm for two neighbouring datasets $X \sim X^\prime$.
  Define the $\ell_2$-sensitivity $\Delta_2(M)$ analogously in terms of the $\ell_2$-norm.
  The Laplace mechanism is the output $M(X) + N$, where $N = \Paren{N_1, \ldots, N_d}$, and $\forall j \in [d]$, with $N_j \sim \text{Lap}\Paren{0, \frac{\Delta_1(M) }{\eps}}$, and is $(\eps, 0)$-DP.
  If instead $N \sim \normal\Paren{0, \frac{\Delta^2_2(M) }{2\rho} \mathbb{I}_{d \times d}}$, this is the Gaussian mechanism, which satisfies $\rho$-CDP.
 \end{lemma}


\subsection{Optimization Preliminaries}
We require the following standard set of optimization preliminaries.

\begin{definition}  A function $f : \mathcal{W} \rightarrow \mathbb{R}$ is $l$-Lipschitz if for all $w_1,w_2 \in \mathcal{W}$ we have $|f(w_1)- f(w_2)| \leq l \cdot  \| w_1-w_2 \|_2$.
\end{definition}

\begin{definition} A function $f$ is $\sc$-strongly convex on $\mathcal{W}$ if for all $w_1,w_2 \in \mathcal{W}$ we have $f(w_1) \geq f(w_2) + \langle  \nabla f(w_2), w_1-w_2 \rangle + \frac{\sc}{2}\| w_1-w_2 \|^2_2$. 
\end{definition}

\begin{definition} A function $f$ is $\sm$-smooth on $\mathcal{W}$ if for all $w_1,w_2 \in \mathcal{W}$, $f(w_1) \leq f(w_2) + \langle  \nabla f(w_2), w_1-w_2 \rangle + \frac{\sm}{2}\| w_1-w_2 \|^2_2$. 
\end{definition}

\begin{definition}  Given a convex set $\mathcal{W}$, we denote the projection of any $\theta \in \mathbb{R}^d$ to the convex set $\mathcal{W}$ by $Proj_{\mathcal{W}}(\theta) = \text{arg} \underset{w \in \mathcal{W}}{\min} \| \theta - w \|_2$.
\end{definition}


\subsection{Problem Setup}

\begin{definition}[Stochastic Convex Optimization (SCO)] Let $\mathcal{D}$ be some unknown distribution over $\mathcal{X}$ and $X = \{x_1,\dots,x_n\}$ be i.i.d.\ samples from $\mathcal{D}$. 
  Given a convex constraint set $\mathcal{W} \subseteq \mathbb{R}^d$ and a convex loss function $\ell : \mathcal{W} \times \mathcal{X} \rightarrow \mathbb{R}$, the goal of \emph{stochastic convex optimization (SCO)} is to find a minimizer $w^{priv}$ for the population risk ${L}_{\mathcal{D}}(w^{priv}) = \mathbb{E}_{x \sim \mathcal{D}}[\ell(w^{priv},x)]$. The utility of an algorithm $\mathcal{A}$ is measured by the {\it expected }excess population risk
$$
\underset{X\sim \mathcal{D}^n, \mathcal{A}}{\mathbb{E}} \left[{L}_{\mathcal{D}}(w^{priv}) - \underset{w\in \mathcal{W}}{\min}{L}_{\mathcal{D}}(w)\right].
$$
\end{definition}

%
\newhz{
We use the following coordinate-wise definition of bounded moments, identical to that of~\cite{WangXDX20}. 
\begin{definition} 
\label{def:bounded-moments}
Let $\mathcal{D}$ be a distribution over $\mathbb{R}^d$ with mean $\mu$. We say that for $k\geq 2$, the $k$-th moment of $\mathcal{D}$ is bounded by $\gamma$, if for every $j \in [d]$,
$
\mathbb{E}[| \langle X-\mu, e_j \rangle |^k] \leq \gamma,
$
where $e_j$ is the $j$-th standard basis vector.
\end{definition}
}

Let $B_R(\vec{c}) \subset \mathbb{R}^d$ be the ball of radius $R > 0$ centered at $\vec{c} \in \mathbb{R}^d$. All our theorems {rely on} the following set of assumptions.



\begin{assumption} 
  \label{assn:main-assumption}
  We assume the following:

1. The loss function $\ell(w,x)$ is non-negative, differentiable and convex for all $w \in \mathcal{W}$ and $x \in \mathcal{X}$.

  2. \newhz{For any $x \in \mathcal{X}$,  $\ell(w,x)$ is $\sm$-smooth on $\mathcal{W}$}.

3. The constraint set $\mathcal{W}$ is bounded with diameter $M$. 

4.   The gradient of the loss function at the optimum is zero.

5.  For any $w \in \mathcal{W}$, the distribution of the gradient of the loss function has bounded $k$-th moments for some $k \geq 2$: $ \nabla \ell(w,x) \sim \mathcal{P}$ satisfies Definition~\ref{def:bounded-moments} with $\gamma = 1$. 

6. For any $w \in \mathcal{W}$, the distribution of the gradient has bounded mean: $\mathbb{E}[\nabla \ell(w,x)] \in B_R(\vec 0)$, \hz{where $R = \newhz{O(1)}$.}
\end{assumption}

The first four points in Assumption~\ref{assn:main-assumption} are standard when studying convex learning problems. The fifth is a relaxation of the {\it Lipschitz} condition in non-heavy-tailed SCO problems, in which the gradient is assumed to be uniformly bounded by a constant. While the gradient in our setting is unbounded, it is realistic to assume that the {\it expected} gradient is inside a ball with some radius $R$. 
In fact, packing lower bounds for private mean estimation necessitate such an assumption under most notions of DP~\cite{KarwaV18}.
\hz{As a direct corollary, $\norm{\nabla {L}_{\mathcal{D}}(w)} \le R$, so ${L}_{\mathcal{D}}(w)$ is $R$-Lipschitz.} We note that all these assumptions are explicitly or implicitly assumed in~\cite{WangXDX20}.


\section{A Framework for Stochastic Convex Optimization}

In this section, we present a general framework for private SCO. 
Before diving into the details, we first provide some intuition on how we approach this problem, 
\hzz{via the classic optimization model.}

Let ${L}_{\mD}(\cdot)$ be the {\it expected} loss function we are trying to minimize. 
Although the data $x \sim \mathcal{X}$, the loss function $\ell(\cdot)$, and its gradient $\nabla \ell(\cdot)$ may be heavy-tailed, $L_{\mD}(\cdot)$ is well-behaved: specifically, Assumption~\ref{assn:main-assumption} implies that it is both convex and $R$-Lipschitz. 
Therefore, if $L_{\mD}(\cdot)$ were known, the problem would reduce to a classical convex optimization problem, solvable by gradient descent (GD). 
Of course, ${L}_{\mD}(\cdot)$ is not known to the optimizer, and we can not directly run gradient descent.
Instead, we estimate $\nabla {L}_{\mD}(\cdot)$ from the samples $X$, incurring an additional loss based on the quality of the approximation.

\newhz{There are two approaches to choosing samples used for this estimate in each iteration. 
The first strategy is to choose the entire dataset $X$. 
This breaks independence between the different iterations, so one must argue using uniform convergence to bound the estimation error for all $w \in \cW$ simultaneously. 
The second strategy is to choose disjoint samples for each iteration, which maintains independence between iterations at the cost of less data and thus more error for each iteration.
We adopt the first strategy in our analysis for convex functions, and the second strategy for strongly convex functions. }

\begin{algorithm}[tb]
  \caption{SCO algorithmic framework $\text{SCOF}_{\eta, T, \text{\MeanEst}}(X)$ }
\label{alg:reduce}
\begin{algorithmic}
\State {\bfseries Input:} $X = \{x_i\}_{i=1}^n, x_i \in \mathbb{R}^d$, algorithm \text{\MeanEst}, parameters $\eta, T $
   \State Initialize $w^0 \in \mathcal{W}$
    \For{$t = 1,2,\dots,T$}
		\If{$\ell$ is \textit{convex}}
		\State \newhz{Select $S_t = X$}
		\ElsIf{$\ell$ is \textit{strongly convex}}
		\State \newhz{Select $S_t = \{x_{(t-1)n/T+1},\ldots, x_{tn/T}\}$}
		\EndIf
		\State \newhz{$\nabla \widetilde{L}_{\mathcal{D}}(w^{t-1}) =  \text{\MeanEst}(\{\nabla \ell(w^{t-1}, x_i)\}_{x_i \in S_t})$}
    		\State $ w^t = \text{Proj}_{\mathcal{W}}(w^{t-1} - \eta_{t-1}\nabla \widetilde{L}_{\mathcal{D}}(w^{t-1}) )$
	\EndFor
	 \State {\bfseries Output:} $\{w^1, w^2, \ldots, w^T\}$
\end{algorithmic}
\end{algorithm}

Our GD-based framework, SCOF, is presented in Algorithm~\ref{alg:reduce}. 
The true gradient $\nabla L_{\mD}(w^{t-1})$ is replaced by an estimate $\nabla \widetilde{L}_{\mD}(w^{t-1})$ obtained by a mean estimation algorithm.

Observe that Algorithm~\ref{alg:reduce} is differentially private if the mean estimator \MeanEst ~is differentially private, a consequence of composition and post-processing of differential privacy. 
In Theorems~\ref{thm:sco-f-c} and~\ref{thm:sco-f-sc}, we quantify the population risk of Algorithm~\ref{alg:reduce} based on the accuracy of \MeanEst.  
Theorem~\ref{thm:sco-f-c} considers convex loss functions, while Theorem~\ref{thm:sco-f-sc} achieves better rates when the loss function is strongly convex. 
Although the proof techniques resemble those in previous work, e.g.,~\cite{AgarwalSYKM18}, we include the analysis for completeness, with proofs in Appendix~\ref{proof_c} and Appendix~\ref{proof_sc}.

\begin{theorem}[Convex] 
\label{thm:sco-f-c}
Suppose that \text{\MeanEst} guarantees that for any $w \in \mathcal{W}$, $\| \mathbb{E}[\nabla\widetilde{L}_{\mathcal{D}}(w)] - \nabla L_{\mathcal{D}}(w) \|_2 \leq B$ and $\mathbb{E}[\| \nabla\widetilde{L}_{\mathcal{D}}(w) - \nabla L_{\mathcal{D}}(w) \|_2^2] \leq G^2$.
Under Assumption \ref{assn:main-assumption}, for any $\eta > 0$ the output $\priv = \frac1T \sum_{t \in [T]} w^t$ produced by SCOF satisfies
\[
\underset{X \sim \mathcal{D}^n,\mathcal{A}}{\mathbb{E}} \left[L_{\mD}(\priv) - L_{\mD}(w^*)\right] \leq \frac{M^2}{2\eta T} + \newshz{ \eta( R^2+G^2) }+ M B,
\]
where $w^* = \arg\min_w L_{\mathcal{D}}(w)$.
\end{theorem}


\begin{theorem}[Strongly convex] 
  \label{thm:sco-f-sc}
Suppose that \text{\MeanEst} guarantees that, for any $w^t, t \in [T]$, $\mathbb{E} [ \| \nabla\widetilde{L}_{\mathcal{D}}(w^t) - \nabla L_{\mathcal{D}}(w^t) \|_2^2] \leq G^2$.
  Under Assumption~\ref{assn:main-assumption}, and the further assumption that the population risk is $\sc$-strongly convex, if $\eta = \frac{1}{\sc+\sm}$,  the output $\priv = w^T$ produced by SCOF satisfies
\begin{align*}
{\underset{X \sim \mathcal{D}^n,\mathcal{A}}{\mathbb{E}}} [L_{\mathcal{D}}(\priv)- L_{\mathcal{D}}(w^*)]  \le \left(1-\frac{\sc \sm}{(\sc + \sm)^2}\right)^TM +\frac{(\sc+\sm)G}{\sc \sm}.
\end{align*}
Specifically, if $ T = {\log\left( \frac{(\lambda + L)^2G^2} { 2\lambda L (\lambda + L)^2-\lambda^2L^2} \right) / \Paren{ 2 \log\Paren{\frac{\lambda^2 +L^2 +\lambda L}{(\sc+\sm)^2}} }}$, the output $\priv$ satisfies
\[
{\underset{X \sim \mathcal{D}^n,\mathcal{A}}{\mathbb{E}}} \left[ L_{\mathcal{D}}(\priv) - L_{\mathcal{D}}(w^*) \right] \leq   \frac{T(\lambda + L)^2 (M^2+1) G^2 }{4\lambda  (\lambda + L)^2-2\lambda^2L},
\]
where $w^* = \arg\min_w L_{\mathcal{D}}(w)$.

\end{theorem}

\section{Mean Estimation Oracle}
\label{MEoralcle}

%
%
%
\newhz{
  We employ an adaption of~\cite{KamathSU20}'s CDPHDME algorithm, which privately estimates the mean of a heavy-tailed distribution.
  In order to improve upon the bounds one would obtain via a black-box application, we provide a novel analysis for an adapted version of this algorithm, which \textcolor{black}{differs} from theirs in two crucial aspects.
}

First, their analysis only applies for a specific choice of the truncation parameter ($\tau$ in Theorem~\ref{thm:main-variation}), which is selected to be optimal for their one-round algorithm. 
However, note that our problem is different from theirs, since GD requires multiple steps instead of only one round. 
If we naively follow the same parameter setting as they do, we will get a loose bound on the excess risk. 
Therefore, we generalize their analysis to accommodate a range of values for $\tau$ to fit our needs.

\newhz{
  Second, while their analysis provides $\ell_2$-error guarantees, Theorem~\ref{thm:sco-f-c} necessitates bounds on the bias and variance of the estimator.
  We thus modify steps in their algorithm to reduce the variance (while leaving the bias untouched).
}

\newhz{Our new analysis can be summarized by Theorem~\ref{thm:main-variation}, where we provide theoretical guarantees for our estimators in both the private and non-private settings. 
We defer our algorithm and the proof to Appendix~\ref{app:proof_oracle}.
}
\newhz{
\begin{theorem}
\label{thm:main-variation}
  Let $\cD$ be a distribution over $\mathbb{R}^d$ with mean $\mu \in \ball{\vec{0}}{R}$ with $R \le 10$ and $k$-th moment bounded by 1. For any $\tau\ge10$ \textcolor{black}{and a universal constant $C \geq 14$}, there exists a polynomial-time (non-private) algorithm \gk{(Algorithm~\ref{alg:cdphdme})} that takes 
$n$ samples from $\cD$, and outputs $\widehat{\mu} \in \mathbb{R}^d$, such that with probability at least $1-\beta$,
\[
\norm{\hat{\mu} - \mu} =  O\Paren{\sqrt{d} \cdot \Paren{\sqrt{\frac{\log\Paren{\frac{d}{\beta}}}{n}} + \Paren{\frac{C}{\tau}}^{k-1}}}.
\]
A $\rho$-CDP adaption, $\text{CDPCWME}(\rho, \tau)$, takes 
  $n$ samples from $\cD$, and outputs $\widetilde{\mu} =  \hat{\mu} + \normal\Paren{0, \frac{\textcolor{black}{1152}\tau^2 d \log^2\Paren{\frac{2d}{\beta}}}{\rho n^2} \mathbb{I}_{d \times d}}$, \newhzz{where $\hat{\mu}$ is the non-private output} \gk{of Algorithm~\ref{alg:cdphdme}}, such that with probability at least $1-\beta$,
\begin{align*}
\norm{\widetilde{\mu} - \mu} =  O\left(\sqrt{d} \cdot \Paren{\sqrt{\frac{\log\Paren{\frac{d}{\beta}}}{n}} + \Paren{\frac{C}{\tau}}^{k-1}}  
+\frac{\tau \log\Paren{\frac{d}{\beta}} \sqrt{d}}{\sqrt{\rho}n}\Paren{\sqrt{d}+\sqrt{\log\frac1{\beta}}} \right).
\end{align*}
Finally, an $(\eps,0)$-DP adaption, $\text{DPCWME}(\eps, \tau)$, takes 
  $n$ samples from $\cD$, and outputs $\widetilde{\mu}$ with $\widetilde{\mu}_j =  \hat{\mu}_j + \text{Lap}\Paren{0, \frac{48\tau d \log\Paren{\frac{2d}{\beta}}}{\eps n}}$, \newhzz{where $\hat{\mu}$ is the non-private output} \gk{of Algorithm~\ref{alg:cdphdme}}, such that with probability at least $1-\beta$,
\begin{align*}
\norm{\widetilde{\mu} - \mu} 
=  O\Paren{\sqrt{d} \cdot \Paren{\sqrt{\frac{\log\Paren{\frac{d}{\beta}}}{n}} + \Paren{\frac{C}{\tau}}^{k-1}} + \frac{\tau d^{\frac{3}{2}}  \log^2\Paren{\frac{d}{\beta}}}{\eps n}}.
\end{align*}
\end{theorem}
}

\newhz{
  Setting $\beta = \frac1{10}$, $\tau = \Paren{\frac{\sqrt{\rho}n}{\sqrt{d}}}^{\frac1{k}}$ for $\text{CDPCWME}(\rho, \tau)$, and  $\tau = \Paren{\frac{\eps n}{d}}^{\frac1{k}}$ for $\text{DPCWME}(\eps, \tau)$, gives the following corollary.
\begin{corollary}
\label{cor:main-variation}
Let $\cD$ be a distribution over $\mathbb{R}^d$ with mean $\mu \in \ball{\vec{0}}{R}$ with $R \le 10$ and $k$-th moment bounded by 1. There exists a polynomial-time $\rho$-CDP algorithm $\text{CDPCWME}\Paren{\rho, \Paren{\frac{\sqrt{\rho}n}{\sqrt{d}}}^{\frac1{k}}}$ that takes 
$n$ samples from $\cD$, and outputs $\widetilde{\mu} \in \mathbb{R}^d$, such that with probability at least $0.9$,
\[
\norm{\widetilde{\mu} - \mu} =  \widetilde{O}\Paren{\sqrt{\frac{d}{n}}+ \sqrt{d} \cdot  \Paren{\frac{\sqrt{d}}{\sqrt{\rho}n}}^{\frac{k-1}{k}}}.
\]
Furthermore, there exists a polynomial-time $\eps$-DP algorithm $\text{DPCWME}\Paren{\eps, \Paren{\frac{\eps n}{d}}^{\frac1{k}}}$ that takes 
$n$ samples from $\cD$, and outputs $\widetilde{\mu} \in \mathbb{R}^d$, such that with probability at least $0.9$,
\[
\norm{\widetilde{\mu} - \mu} =  \widetilde{O}\Paren{\sqrt{\frac{d}{n}}+ \sqrt{d} \cdot  \Paren{\frac{d }{\eps n}}^{\frac{k-1}{k}}}.
\]
\end{corollary}
}

\section{Algorithms for SCO with Heavy-Tailed Data}

In this section we introduce our main algorithm for $\rho$-CDP SCO, Algorithm~\ref{alg:a-dpscoht}. 
We analyze utility for convex loss functions in Section~\ref{sec:ub_accuracy_convex}, while the results for strongly convex loss functions are in Section~\ref{sec:ub_accuracy_sconvex}.



\begin{algorithm}[H]
\caption{CDP-SCO algorithm with heavy-tailed data }
\label{alg:a-dpscoht}
\begin{algorithmic}[1]
\State {\bfseries Input:} ${X = \{x_i\}_{i=1}^n, x_i \in \mathbb{R}^d}$, parameters $\eta,\xl{\rho}, T $
\State $\{ w^t\}_{t=1}^T = SCOF_{\eta, T, \MeanEst(\rho/T)}(X)$ 
\State {\bfseries Output:} \hz{$\priv$ }
\end{algorithmic}
\end{algorithm}

Privacy is straightforward: since each of the $T$ steps of the algorithm is $\rho/T$-CDP, composition of CDP (Lemma~\ref{lem:composition}) gives the following privacy guarantee.
\begin{lemma}
\label{lem:lemma13}
  Algorithm~\ref{alg:a-dpscoht} is $\rho$-CDP.
\end{lemma}



\subsection{Convex Setting}
\label{sec:ub_accuracy_convex}

In this section, we consider convex and smooth loss functions.
We provide accuracy guarantees for Algorithm~\ref{alg:a-dpscoht} in the following two theorems, each of which instantiates our framework with a different mean estimation oracle.
The proofs follow by appropriately selecting the truncation parameter $\tau$, and balancing the bias and variance in SCOF. 
We defer the proofs to Appendix~\ref{add2ub2} and  Appendix~\ref{appendix:proof_k_2}, respectively.

\begin{theorem}[Convex]\htodo{assumption on smoothness, $R$, and $n,d$}
\label{cor:corollary15}
  Suppose we have a stochastic convex optimization problem which satisfies Assumption~\ref{assn:main-assumption}. Assuming $R \le 10$, $\sm \le 10$,
 Algorithm~\ref{alg:a-dpscoht}, instantiated with CDPCWME with parameters \hz{$T = \frac{R^2\rho n^2}{\tau^2 d^4}$}, $\eta = \frac{M}{R \sqrt{T}}$, and $\tau = \Paren{\frac{\sqrt{\rho} n}{Md^{\frac{3}{2}}}}^{\frac1k}$, outputs \hz{$\priv =  \frac1T \sum_{t \in [T]} w^t$}, such that
\begin{align*}
\underset{X\sim \mathcal{D}^n, \mathcal{A}}{\mathbb{E}}[ L_{\mathcal{D}}(\priv) - L_{\mathcal{D}}(w^*) ]  
\leq  O\Bigg( \frac{Md}{\sqrt{n}}+\frac{Md^2}{n\sqrt{\rho}}\cdot \Paren{\frac{\sqrt{\rho}n}{Md^{\frac{3}{2} }}}^{\frac1{k}} \Bigg),
\end{align*}
  where $w^* = \arg\min_w L_{\mathcal{D}}(w)$, and $M$ is the diameter of the constraint set $\cW$. \newhzz{The running time is $O(ndT)$\footnote{Suppose $M$ and $R$ are constants, this bound is vacuous unless the loss is $O(1)$, i.e., we implicitly require $n$ is large enough such that the denominator is larger (in order) than the numerator. The same constraint applies to the other theorems.}
.}
  \end{theorem}
  
 \begin{remark}
  Our non-standard choice of the truncation parameter $\tau$ in Theorem~\ref{cor:corollary15} is crucial to obtaining our results.
 If one were to na\"ively adopt $\tau = \Paren{\frac{\sqrt{\rho} n}{d^{\frac{3}{2}}\sqrt{T}}}^{\frac1{k}}$ to balance the bias and standard deviation for each iteration, we would achieve much worse bounds. 
   Instead, in order to reduce bias we truncate far less aggressively than done in~\cite{KamathSU20}, which comes at the cost of increased variance.
  For example, considering the case when $d=1$, if we were to use the choice of $\tau = \Paren{\frac{\sqrt{\rho} n}{\sqrt{T}}}^{\frac1{k}}$ for the convex case, the error would be $O\Paren{\frac1{\sqrt{T}} + \Paren{\frac1{\tau}}^{k-1}}=O\Paren{\frac1{\sqrt{T}} +  \Paren{\frac{\sqrt{T}}{\sqrt{\rho} n} }^{\frac{k-1}{k}} }$.
  Fixing $T = \Paren{\sqrt{\rho}n }^{\frac{2k-2}{2k-1}}$, we obtain the bound of  $O\Paren{\Paren{ \frac{1}{\sqrt{\rho}\ns}}^{\frac{k-1}{2k-1}}}$
 instead of our bound of $O\Paren{\Paren{ \frac{1}{\sqrt{\rho}\ns}}^{\frac{k-1}{k}}}$ in Theorem~\ref{cor:corollary15}.
  In the limit as $k \rightarrow \infty$, our bound is quadratically better.
\end{remark}

\begin{remark}
Theorem~\ref{cor:corollary15} is written in terms of the expectation over the randomness of samples and algorithms. However, it can be easily generalized to the high-probability setting. We present the high-probability version of Theorem~\ref{cor:corollary15} in Appendix~\ref{sec:proof_hp}.
\end{remark}

Alternatively, one can adopt the mean estimation oracle of~\cite{Holland19}, as done by~\cite{WangXDX20}. 
However, we provide a more careful analysis, resulting in a significantly improved error rate.
We provide further details of the mean estimation oracle and a proof of the following theorem in Appendix~\ref{appendix:proof_k_2}.

%
\begin{theorem}[Convex]
\label{thm:ub_sco_k_2}
 Suppose we have a stochastic convex optimization problem which satisfies Assumption~\ref{assn:main-assumption}. Assuming $R \le 10$ and $\sm \le 10$, for any $0.5\le \bal \le 2$,
 Algorithm~\ref{alg:a-dpscoht}, instantiated with CDPNSME (Algorithm~\ref{alg:holland}) with parameters \hz{$T = \frac{R^2\rho n^2}{\tau^2 d^2}$}, $\eta = \frac{M}{R \sqrt{T}}$, and $\tau = \Paren{\frac{\sqrt{\rho} n}{Md^{\bal }}}^{\frac1{2}}$, outputs \hz{$\priv =  \frac1T \sum_{t \in [T]} w^t$}, such that 
\begin{align*}
\underset{X\sim \mathcal{D}^n, \mathcal{A}}{\mathbb{E}}[ L_{\mathcal{D}}(\priv) - L_{\mathcal{D}}(w^*) ] 
 \leq  O \Paren{ \frac{\sqrt{M}d^{\frac{3-\bal}{2}}}{\sqrt{n}}+ \frac{\sqrt{M} d^{\frac{1+\bal}{2} }} {\rho^{\frac14}\sqrt{n}}   },
\end{align*}
  where $w^* = \arg\min_w L_{\mathcal{D}}(w)$, and $M$ is the diameter of the constraint set $\cW$. \newhzz{The running time is $O(ndT)$.}
  \end{theorem}

\begin{remark}
\newhz{
By varying $q$ from $0.5$ to $2$, Theorem~\ref{thm:ub_sco_k_2} is able to achieve different balances between the non-private (first) and private (second) error terms. 
In practice, one should choose $q$ which minimizes their sum, which depends on the instance parameters. 
  When $q=0.5$, our bound is $O\Paren{ \frac{\sqrt{M}d^{\frac{5}{4}}}{\sqrt{n}}+ \frac{\sqrt{M} d^{\frac{3}{4} }} {\rho^{\frac14}\sqrt{n}}}$, where the private term matches the $k = 2$ lower bound in Theorem~\ref{thm:lower-bound_c}.
Additionally, when $q=1$ and $\rho$ and $M$ are constants, our bound is $O\Paren{ \frac{d} {\sqrt{n}}}$, strictly improving upon $O\Paren{\frac{d}{ n^{\frac13}}}$ in~\cite{WangXDX20}.
}
\end{remark}

Combining Theorems~\ref{cor:corollary15} and \ref{thm:ub_sco_k_2} gives the convex part of Theorem~\ref{thm:main-ub-informal}.

\subsection{Strongly Convex Setting}
\label{sec:ub_accuracy_sconvex}
In this section, we consider strongly convex and smooth loss functions.
The analysis is somewhat simpler than the convex case, due to number of iterations being only logarithmic.
The proof of the following theorem is in Appendix~\ref{add2ub1}.

\begin{theorem}[Strongly Convex]
\label{cor:corollary14}
  Suppose we have a stochastic convex optimization problem which satisfies Assumption~\ref{assn:main-assumption}, and additionally, the loss function $\ell$ is $\sc$-strongly convex.
  Algorithm~\ref{alg:a-dpscoht}, instantiated with CDPCWME with parameters \hz{$T = {\log\left(\frac{(\sc+\sm)G}{\sc \sm}\right) / \log\Paren{\frac{\sc^2+\sm^2+\sc \sm}{(\sc+\sm)^2}} }$} with $G = \widetilde{O} \left(\sqrt{\frac{d}{ n}} + \sqrt{d} \cdot \left(\frac{\sqrt{d}}{\sqrt{\rho} n}\right)^{\frac{k-1}{k}} \right) $, $\eta = \frac{1}{\sc+\sm}$ and $\tau = \Paren{\frac{\sqrt{\rho} n}{ \sqrt{d}T^{\frac{3}{2}}}}^{1/k}$, outputs \hz{$\priv =  w^T$}, such that
  \begin{align*}
\underset{X\sim \mathcal{D}^n, \mathcal{A}}{\mathbb{E}}[  L_{\mathcal{D}}(\priv) - L_{\mathcal{D}}(w^*) ] 
 \leq \frac{(M+1)^2 (\sc+\sm)^2}{\sc^2\sm} \cdot \tilde O\left(\frac{d}{ n} + d\cdot \left(\frac{\sqrt{d}}{\sqrt{\rho} n}\right)^{\frac{2k-2}{k}} \right),
\end{align*}
  where $w^* = \arg\min_w L_{\mathcal{D}}(w)$. \newhzz{The running time is $O(ndT)$.}
\end{theorem}

\begin{remark}
Although in Theorem~\ref{cor:corollary15} and~\ref{cor:corollary14}, we provide our utility guarantees in terms of the expectation, they can be easily generalized to the high-probability setting. In Appendix~\ref{sec:proof_hp}, we present the high-probability version of Theorem \ref{cor:corollary15} as an example.
\end{remark}

\section{Lower Bounds for DP SCO with Heavy-Tailed Data}
\label{sec:lb}

In this section, we present our lower bounds for $\rho$-CDP SCO. 
Our results are generally attained by reducing from mean estimation to SCO, where
similar connections have been explored when proving lower bounds for DP empirical risk minimization~\cite{BassilyST14}.
In order to prove some of our lower bounds, we introduce a new technical tool, a CDP version of Fano's inequality (Theorem~\ref{thm:dp_fano}), which is of independent interest.


\subsection{Strongly-Convex Loss Functions}
\label{sec:lb_sc}

\begin{theorem}[Strongly convex case] 
\label{thm:lower-bound_sc}
  Let $n,d \in \mathbb{N}$ and $\rho > 0$. There exists a strongly convex and smooth loss function $\ell: \mathcal{W} \times \mathbb{R}^d$, such that for every $\rho$-CDP algorithm $\mathcal{A}$ (whose output on input $X$ is denoted by $w^{priv} = \mathcal{A}(X)$), 
there exists a distribution $\mathcal{D}$ on $\mathbb{R}^d$ such that $\forall w \in \mathcal{W}$, $\sup_{j \in [d]}  \expectsub{ \absv{ \langle \nabla \ell(w,x)-\expect{\nabla \ell(w,x)}, e_j \rangle}^k } {x \sim \mathcal{D}} \le 1$ ($e_j$ is the $j$-th standard basis vector), which satisfies
\begin{align*}
\underset{X\sim \mathcal{D}^n, \mathcal{A}}{\mathbb{E}}[  L_{\mathcal{D}}(w^{priv}) - L_{\mathcal{D}}(w^*)) ] 
  \ge \Omega\Paren{ \frac{d}{n} + d \cdot \min\Paren{1, \left(\frac{\sqrt{d}}{\sqrt{\rho} n}\right)^{\frac{2k-2}{k}}}},
\end{align*}
where $w^* = \arg\min_w L_{\mathcal{D}}(w)$.
\end{theorem}

\begin{proof}
The following lemma shows a reduction from mean estimation to SCO.
The proof is deferred to Appendix~\ref{lb1}.

\begin{lemma} 
\label{lem:reduction_c}
Let $n,d \in \mathbb{N}$, and $\rho > 0$. There exists a strongly convex and smooth loss function $\ell: \mathcal{W} \times \mathbb{R}^d$, such that
for every $\rho$-CDP algorithm $\mathcal{A}$ (whose output on input $X$ is denoted by $w^{priv} = \mathcal{A}(X)$), and every distribution $\mathcal{D}$ on $\mathbb{R}^d$ with $\mathbb{E}[\mathcal{D}]=\mu$, 
\begin{align*}
\underset{X\sim \mathcal{D}^n, \mathcal{A}}{\mathbb{E}}[  L_{\mathcal{D}}(w^{priv}) - L_{\mathcal{D}}(w^*) ]  
= \underset{X\sim \mathcal{D}^n, \mathcal{A}}{\mathbb{E}} \left[\frac12\| w^{priv} - \mu \|^2_2\right],
\end{align*}
where $w^* = \arg\min_w L_{\mathcal{D}}(w)$.
%
\end{lemma}


The following lemma provides lower bounds for mean estimation, under both DP and CDP.
The first term is the non-private sample complexity, and is folklore for Gaussian mean estimation.
To prove the second term, we leverage our CDP version of Fano's inequality (Theorem~\ref{thm:dp_fano}), based on the packing of distributions employed by~\cite{BarberD14}.
Detail are in Appendix~\ref{sec:proof_lb_me}.

\begin{lemma}
\label{lem:duchi_res}
Let $n,d \in \mathbb{N}$ and $\rho > 0$. 
For every $\rho$-CDP algorithm $\mathcal{A}$, there exists a distribution $\mathcal{D}$ on $\mathbb{R}^d$ with $\mathbb{E}[\mathcal{D}]=\mu$ and $\sup_{j \in [d]} \mathbb{E}_{x \sim \mathcal{D} }[\absv{\langle e_j, x-\mu \rangle}^k] \leq 1$ ($e_j$ is the $j$-th standard basis vector), such that
\begin{align*}
\underset{X\sim \mathcal{D}^n, \mathcal{A}}{\mathbb{E}} [\| \mathcal{A}(X) - \mu \|_2] 
 \geq \Omega \left(\sqrt{\frac{d}{n}} + \sqrt{d} \cdot \min \left(1, \left(\frac{\sqrt{d}}{\sqrt{\rho} n}\right)^{\frac{k-1}{k}}\right)\right).
\end{align*}
Additionally, for every $(\eps,0)$-DP algorithm $\mathcal{A}$, there exists a distribution $\mathcal{D}$ on $\mathbb{R}^d$ with $\mathbb{E}[\mathcal{D}]=\mu$ and $\sup_{j \in [d]} \mathbb{E}_{x \sim \mathcal{D} }[\absv{\langle e_j, x-\mu \rangle}^k] \leq 1$ ($e_j$ is the $j$-th standard basis vector), such that
\begin{align*}
\underset{X\sim \mathcal{D}^n, \mathcal{A}}{\mathbb{E}} [\| \mathcal{A}(X) - \mu \|_2] 
\geq \Omega \left(\sqrt{\frac{d}{n}} + \sqrt{d} \cdot \min \left(1, \left(\frac{{d}}{\eps n}\right)^{\frac{k-1}{k}}\right)\right).
\end{align*}
\end{lemma}

Observe that by Jensen's inequality, for $\rho$-CDP algorithms,
\begin{align*}
\underset{X\sim \mathcal{D}^n, \mathcal{A}}{\mathbb{E}} [\| \mathcal{A}(X) - \mu \|_2^2] 
\geq \Omega \left({\frac{d}{n}} + {d} \cdot \min \left(1, \left(\frac{\sqrt{d}}{\sqrt{\rho} n}\right)^{\frac{2k-2}{k}}\right)\right).
\end{align*}
Combining Lemma~\ref{lem:reduction_c} and Lemma~\ref{lem:duchi_res} yields Theorem~\ref{thm:lower-bound_sc}.
\end{proof}

\subsection{Convex Loss Functions}
\label{sec:lb_c}
%

The convex case is slightly different from the strongly convex case, as it can not be reduced to mean estimation in a black-box fashion.
As before, we apply our CDP version of Fano's inequality (Theorem~\ref{thm:dp_fano}), based on the packing of distributions employed by~\cite{BarberD14}.
The proof appears in Appendix \ref{lb2}.

\begin{theorem}[Convex case] 
\label{thm:lower-bound_c}
  Let $n,d \in \mathbb{N}$ and $\rho > 0$. There exists a convex and smooth loss function $\ell$, such that for every $\rho$-CDP algorithm $\mathcal{A}$ (whose output on input $X$ is denoted by $w^{priv} = \mathcal{A}(X)$), 
there exists a distribution $\mathcal{D}$ on $\mathbb{R}^d$ such that 
$\forall w \in \mathcal{W}$, $\sup_{j \in [d]}  \expectsub{ \absv{ \langle \nabla \ell(w,x)-\expect{\nabla \ell(w,x)}, e_j \rangle}^k } {x \sim \mathcal{D}} \le 1$ ($e_j$ is the $j$-th standard basis), which satisfies
\begin{align*}
\underset{X\sim \mathcal{D}^n, \mathcal{A}}{\mathbb{E}}[  L_{\mathcal{D}}(w^{priv}) - L_{\mathcal{D}}(w^*)) ]  
  \geq \Omega\Paren{ \sqrt{\frac{d}{n}} + \sqrt{d} \cdot \min\Paren{1, \left(\frac{\sqrt{d}}{\sqrt{\rho} n}\right)^{\frac{k-1}{k}}}},
\end{align*}
where $w^* = \arg\min_w L_{\mathcal{D}}(w)$.
\end{theorem}

\section{Acknowledgements}
We thank Andrew Lowy, Daniel Levy, Meisam Razaviyayn, Ziteng Sun, and an anonymous reviewer for pointing out issues in a previous version of our paper.

\bibliography{biblio}
\bibliographystyle{alpha}

\newpage
\appendix

\section{Useful Inequalities}

\begin{lemma}
\label{Cheby}
Let $\mathcal{D}$ be a distribution over $\mathbb{R}$ with mean $\mu$, and $k$-th moment bounded by $\gamma$. Then the following holds for any $a>1$. 
$$
\mathbb{\underset{X\sim \mathcal{D}}{P}}[|X-\mu|>a\gamma^{\frac{1}{k}}] \leq \frac{1}{a^k}.
$$
\end{lemma}

The following lemma comes from~\cite{KamathSU20}. We prove it here for completeness.
\begin{lemma}
\label{lem:mean_concentration}
Let $\mathcal{D}$ be a distribution over $\mathbb{R}$ with mean $\mu$, and $k$-th moment bounded by 1.  Suppose $X_1, \ldots, X_n$ are generated from $\mathcal{D}$, then with probability at least $0.99$,
$$
\absv{\frac1n \sum_{i=1}^\ns X_i - \mu} \le \frac{10}{\sqrt{n}}.
$$
\end{lemma}

\begin{proof}
By Jensen's inequality,
\[
\expectation{\Paren{X - \mu}^2} \le \expectation{\absv{X - \mu}^k}^{\frac{2}{k}} \le 1.
\]
Then
\begin{align*}
\expectation{\Paren{\frac1n \sum_{i=1}^\ns X_i - \mu}^2} &= \frac1{n^2} \expectation{\Paren{\sum_{i=1}^\ns X_i - n \mu}^2} \\
& =  \frac1{n^2} \expectation{\sum_{i=1}^\ns \Paren{X_i -  \mu}^2} \le \frac1{n}.
\end{align*}

By Chebyshev's inequality,
\begin{align*}
\probof{ \absv{\frac1n \sum_{i=1}^\ns X_i - \mu} \ge \frac{10}{\sqrt{\ns}} } \le \frac1{100}.
\end{align*}
\end{proof}




%

\section{Omitted Proofs}
\label{proofs}
\allowdisplaybreaks

\subsection{Proof of Theorem \ref{thm:sco-f-c}}
\label{proof_c}


We let $L_{\mathcal{D}}(w^t) = \underset{x \sim \mathcal{D}}{\mathbb{E}}[  \ell(w^t, x)]$. By Assumption~\ref{assn:main-assumption}, for all $t$, 
\begin{align*}
\| \nabla  L_{\mathcal{D}}(w^t) \|_2 = \left\| \nabla  \underset{x \sim \mathcal{D}}{\mathbb{E}}[  \ell(w^t, x)] \right\|_2  = \left\|  \underset{x \sim \mathcal{D}}{\mathbb{E}}[  \nabla  \ell(w^t, x)] \right\|_2 \leq R.
\end{align*}

Let ${w'}^t = w^{t-1}  - \eta \nabla \widetilde{L}_{\mathcal{D}}(w^{t-1})$, and $w^t$ denotes its projection to $\mathcal{W}$.
By the convexity of $L_{\mathcal{D}}(\cdot)$ (see, e.g., Section 14.1.1 in \cite{Shai-Shai}), we have
\begin{align}
&\underset{\mathcal{A},X \sim \mathcal{D}^n}{\mathbb{E}} \left[L_{\mathcal{D}}(w^{priv}) - L_{\mathcal{D}}(w^*)\right] \nonumber\\
= & \underset{\mathcal{A},X \sim \mathcal{D}^n}{\mathbb{E}} \left[L_{\mathcal{D}}\left(\frac{1}{T}\sum_{t=1}^T w^{t}\right) - L_{\mathcal{D}}(w^*)\right] \nonumber\\
\leq & \underset{\mathcal{A},X \sim \mathcal{D}^n}{\mathbb{E}} \left[\frac{1}{T}\sum_{t=1}^T \left( L_{\mathcal{D}}\left( w^{t}\right) \right) - L_{\mathcal{D}}(w^*)\right] \label{equ:fr_c_JS}\\
= & \underset{\mathcal{A},X \sim \mathcal{D}^n}{\mathbb{E}} \left[\frac{1}{T}\sum_{t=1}^T \left( L_{\mathcal{D}}\left( w^{t}\right) - L_{\mathcal{D}}(w^*) \right) \right] \nonumber\\
\leq  &\underset{\mathcal{A},X \sim \mathcal{D}^n}{\mathbb{E}} \left[ \frac{1}{T}\sum_{t=1}^T \frac{1}{\eta} \left\langle \eta  \nabla L_{\mathcal{D}}(w^t), w^t- w^* \right\rangle \right] \label{equ:fr_c_CV}
\end{align}
where~\eqref{equ:fr_c_JS} is by the Jensen's inequality and \eqref{equ:fr_c_CV} is by the convexity of $L_{\mathcal{D}}$. Continuing the proof, 
\rvm{
\begin{align}
&\underset{\mathcal{A},X \sim \mathcal{D}^n}{\mathbb{E}} \left[L_{\mathcal{D}}(w^{priv}) - L_{\mathcal{D}}(w^*)\right] \nonumber\\
\leq  &\underset{\mathcal{A},X \sim \mathcal{D}^n}{\mathbb{E}} \left[ \frac{1}{T}\sum_{t=1}^T \frac{1}{\eta} \left\langle \eta  \nabla L_{\mathcal{D}}(w^t) +  \eta \nabla \widetilde{L}_{\mathcal{D}}(w^t) -  \eta \nabla \widetilde{L}_{\mathcal{D}}(w^t), w^t- w^* \right\rangle \right] \nonumber\\
\xtl{=}  &\underset{\mathcal{A},X \sim \mathcal{D}^n}{\mathbb{E}} \left[ \frac{1}{T}\sum_{t=1}^T \left\langle \nabla L_{\mathcal{D}}(w^t) -   \nabla \widetilde{L}_{\mathcal{D}}(w^t), w^t- w^* \right\rangle \right] + \underset{\mathcal{A},X \sim \mathcal{D}^n}{\mathbb{E}} \left[ \frac{1}{T}\sum_{t=1}^T \frac{1}{\eta} \left\langle  \eta \nabla \widetilde{L}_{\mathcal{D}}(w^t), w^t- w^* \right\rangle \right]. \nonumber
\end{align}
We bound the first term, note that $\norm{w^t - w^*} \le M$, and $\| \mathbb{E}[\nabla\widetilde{L}_{\mathcal{D}}(w)] - \nabla L_{\mathcal{D}}(w) \|_2 \leq B$,
\begin{align}
\label{equ:convex_first_term}
&\underset{\mathcal{A},X \sim \mathcal{D}^n}{\mathbb{E}} \left[ \frac{1}{T}\sum_{t=1}^T  \left\langle  \nabla L_{\mathcal{D}}(w^t) -  \nabla \widetilde{L}_{\mathcal{D}}(w^t), w^t- w^* \right\rangle \right] \nonumber \\
= &\frac{1}{T}\sum_{t=1}^T  \left\langle \nabla L_{\mathcal{D}}(w^t) -   \underset{\mathcal{A},X \sim \mathcal{D}^n}{\mathbb{E}} \left[ \nabla \widetilde{L}_{\mathcal{D}}(w^t)\right], w^t- w^* \right\rangle
\le BM.
\end{align}
Then we move to the second term.
\begin{align}
&\underset{\mathcal{A},X \sim \mathcal{D}^n}{\mathbb{E}} \left[ \frac{1}{T}\sum_{t=1}^T \frac{1}{\eta} \left\langle  \eta \nabla \widetilde{L}_{\mathcal{D}}(w^t), w^t- w^* \right\rangle \right]\nonumber\\
& = \underset{\mathcal{A},X \sim \mathcal{D}^n}{\mathbb{E}} \left[ \frac{1}{T}\sum_{t=1}^T     \Biggl(   \frac{1}{2\eta}    \biggl(   - \left\| w^t - w^* -\eta \nabla \widetilde{L}_{\mathcal{D}}(w^t)  \right\|^2 + \left\| w^t - w^* \right\|^2    \biggr)   + \frac{\eta}{2} \left\| \nabla  \widetilde{L}_{\mathcal{D}}(w^t) \right\|^2    \Biggr)    \right] \label{equ:fr_c_1} \\
& = \frac{1}{T}\sum_{t=1}^T\left( \frac{1}{2\eta} \left(-  \mathbb{E}\left[ \left\| {w'}^{t+1} - w^* \right\|^2 \right]  + \mathbb{E}\left[  \left\| {w}^t - w^* \right\|^2 \right] \right) +  \frac{\eta}{2} \cdot \mathbb{E}\Brack{ \left\| \nabla  \widetilde{L}_{\mathcal{D}}(w^t) \right\|^2} \right) \label{equ:fr_c_3} \\
& \leq \frac{1}{T}\sum_{t=1}^T\left( \frac{1}{2\eta} \left(-  \mathbb{E}\left[ \left\| w^{t+1} - w^* \right\|^2 \right]  + \mathbb{E}\left[  \left\| {w}^t - w^* \right\|^2 \right] \right) +  \frac{\eta}{2} \cdot \mathbb{E}\Brack{ \left\| \nabla  \widetilde{L}_{\mathcal{D}}(w^t) \right\|^2} \right) \label{equ:fr_c_4}\\
& \xtl{=} \frac{1}{2\eta T}\left(- \mathbb{E}\left[  \left\| w^T - w^* \right\|^2 \right]+ \mathbb{E}\left[  \left\| w^1 - w^* \right\|^2 \right] \right)  + \frac{\eta}{2T} \cdot \mathbb{E}\Brack{ \sum_{t=1}^T \left\| \nabla  \widetilde{L}_{\mathcal{D}}(w^t) \right\|^2} \label{equ:fr_c_5}\\
& \leq  \frac{M^2}{2\eta T} +\frac{\eta}{2T}  \cdot \mathbb{E}\Brack{ \sum_{t=1}^T \left\| \nabla  \widetilde{L}_{\mathcal{D}}(w^t) \right\|^2}. \label{equ:convex_mid_term}
\end{align}
where~\eqref{equ:fr_c_1} comes from the fact that $\forall a, b \in \mathbb{R}^d$, $\langle a,b \rangle = \frac12\Paren{\norm{a}^2+\norm{b}^2- \|a-b\|^2_2}$, and~\eqref{equ:fr_c_3} is by the updating rule,~\eqref{equ:fr_c_4} comes from the fact that $\norm{{w'}^{t+1}-w^*} \ge \norm{w^{t+1}-w^*}$, and~\eqref{equ:fr_c_5} is by the telescopic sum.
}

\rvm{
Finally, for all $t \in [T]$,
\begin{align}
\label{equ:convex_final_term}
\mathbb{E}\Brack{\left\| \nabla  \widetilde{L}_{\mathcal{D}}(w^t) \right\|^2} &=  \mathbb{E}\Brack{\left\| \nabla  \widetilde{L}_{\mathcal{D}}(w^t) - \nabla {L}_{\mathcal{D}}(w^t)+ \nabla {L}_{\mathcal{D}}(w^t) \right\|^2}  \nonumber\\
&\le { \newshz{2} \mathbb{E}\Brack{\left\| \nabla  \widetilde{L}_{\mathcal{D}}(w^t) - \nabla {L}_{\mathcal{D}}(w^t) \right\|^2}+ \newshz{2} \left\| \nabla {L}_{\mathcal{D}}(w^t) \right\|^2} \nonumber\\
&\le \textcolor{black}{2G^2+2R^2},
\end{align}
where we note that $\mathbb{E}[\| \nabla\widetilde{L}_{\mathcal{D}}(w) - \nabla L_{\mathcal{D}}(w) \|_2^2] \leq  G^2$, and $\left\| \nabla {L}_{\mathcal{D}}(w^t) \right\|^2 \le   R^2$.
}

\rvm{
We conclude the proof by combining~\eqref{equ:convex_first_term}, ~\eqref{equ:convex_mid_term}, and~\eqref{equ:convex_final_term}.
}

\subsection{Proof of Theorem \ref{thm:sco-f-sc}}
\label{proof_sc}

The argument is broadly similar to the proof of Theorem 5 in \cite{WangXDX20}, albeit with some minor modifications.

Let ${w'}^t = w^{t-1} - \eta \nabla \widetilde{L}_{\mathcal{D}}(w^{t-1})$. 
Now we have 
\begin{align*}
\| {w'}^t - w^* \|_2 &= \| w^{t-1} - \eta \nabla \widetilde{L}_{\mathcal{D}}(w^{t-1}) -w^*\|_2 \\
&\leq \|w^{t-1} - \eta \nabla L_{\mathcal{D}}(w^{t-1}) - w^* \|_2 + \eta \|\nabla \widetilde{L}_{\mathcal{D}}(w^{t-1}) - \nabla L_{\mathcal{D}}(w^{t-1}) \|_2. 
\end{align*}
It should be noticed that, the second term is bounded by $\eta G$ in expectation, since $\mathbb{E} [\|\nabla \widetilde{L}_{\mathcal{D}}(w^{t-1}) - \nabla L_{\mathcal{D}}(w^{t-1}) \|_2^2  ] \leq G^2$. 
For the first term, by the coercivity of strongly convex functions (Lemma 3.11 in \cite{Bubeck15}) 
$$
\langle w^{t-1} - w^*,  \nabla L_{\mathcal{D}}(w^{t-1})  \rangle \geq \frac{\sc \sm}{\sc+\sm}\| w^{t-1}-w^*\|^2_2 + \frac{1}{\sc + \sm} \| \nabla L_{\mathcal{D}}(w^{t-1})\|^2_2
$$ 
and by taking $\eta = \frac{1}{\sc+\sm}$ we have 
\begin{align*}
\|w^{t-1} - \eta \nabla L_{\mathcal{D}}(w^{t-1}) - w^* \|_2^2 &= \| w^{t-1} - w^* \|^2_2 + \| \eta \nabla L_{\mathcal{D}}(w^{t-1})\|^2_2 - 2 \langle w^{t-1} - w^*, \eta  \nabla L_{\mathcal{D}}(w^{t-1})  \rangle \\
&\leq \left(1-\frac{2\sc \sm}{(\sc + \sm)^2}\right)\| w^{t-1} - w^* \|^2_2 - \frac{1}{(\sc+\sm)^2}\| \nabla L_{\mathcal{D}}(w^{t-1})\|^2_2\\
&\leq \left(1-\frac{2\sc \sm}{(\sc + \sm)^2}\right)\| w^{t-1} - w^* \|^2_2.
\end{align*}
Now using the inequality $\sqrt{1-x} \leq 1- \frac{x}{2}$ we combine two terms together to have
$$
 \| {w'}^t - w^* \|_2  \leq \left(1-\frac{\sc \sm}{(\sc + \sm)^2}\right) \| w^{t-1} - w^* \|_2 + \frac{\|\nabla \widetilde{L}_{\mathcal{D}}(w^{t-1}) - \nabla L_{\mathcal{D}}(w^{t-1}) \|_2}{\sc+\sm}.
$$
\hz{Recall that $w^t$ is the projection of ${w'}^t$ on $\mathcal{W}$, which implies $\| w^t - w^{*} \|_2 \leq \| {w'}^t -w^{*}\|_2$. Therefore,}
$$
 \| w^t - w^* \|_2  \leq \left(1-\frac{\sc \sm}{(\sc + \sm)^2}\right)  \| w^{t-1} - w^* \|_2 + \frac{\|\nabla \widetilde{L}_{\mathcal{D}}(w^{t-1}) - \nabla L_{\mathcal{D}}(w^{t-1}) \|_2}{\sc+\sm}.
$$
After $T$ multiplications, 
$$
\| w^T -w^* \|_2  \leq \left(1-\frac{\sc \sm}{(\sc + \sm)^2}\right)^T \| w^{0} - w^* \|_2 + \sum_{t=0}^{T-1} \left(1- \frac{\sc \sm}{(\sc+\sm)^2}\right)^t \left[ \frac{\|\nabla \widetilde{L}_{\mathcal{D}}(w^{T-t}) - \nabla L_{\mathcal{D}}(w^{T-t}) \|_2}{\sc+\sm}\right].
 $$
 Therefore, by the Cauchy-Schwarz inequality,
 $$
\| w^T -w^* \|_2^2  \leq T \left(1-\frac{\sc \sm}{(\sc + \sm)^2}\right)^{2T} \| w^{0} - w^* \|_2^2 + T \sum_{t=0}^{T-1} \left(1- \frac{\sc \sm}{(\sc+\sm)^2}\right)^{2t} \left[ \frac{\|\nabla \widetilde{L}_{\mathcal{D}}(w^{T-t}) - \nabla L_{\mathcal{D}}(w^{T-t}) \|_2^2}{(\sc+\sm)^2}\right],
 $$
 and
  $$
\mathbb{E} [\| w^T -w^* \|_2^2]  \leq T  \left(  \left(1-\frac{\sc \sm}{(\sc + \sm)^2}\right)^{2T} M^2 + \sum_{t=0}^{T-1} \left(1- \frac{\sc \sm}{(\sc+\sm)^2}\right)^{2t} \left[ \frac{G^2}{(\sc+\sm)^2}\right]\right).
 $$
Note that $\sum_{t=0}^{T-1} \left(1- \frac{\sc \sm}{(\sc+\sm)^2}\right)^{2t} \le 1/\Paren{1 -  \left(1- \frac{\sc \sm}{(\sc+\sm)^2}\right)^{2}}$,
  $$
\mathbb{E} [\| w^T -w^* \|_2^2]  \leq T  \left(  \left(1-\frac{\sc \sm}{(\sc + \sm)^2}\right)^{2T} M^2 +  \frac{(\lambda + L)^2G^2}{2\lambda L (\lambda + L)^2-\lambda^2L^2}\right).
 $$
Letting $T = {\log\left( \frac{(\lambda + L)^2G^2} { 2\lambda L (\lambda + L)^2-\lambda^2L^2} \right) / \Paren{ 2 \log\Paren{\frac{\lambda^2 +L^2 +\lambda L}{(\sc+\sm)^2}} }}$,
$$ 
\mathbb{E} [ \| w^T - w^* \|_2^2 ] \leq  \frac{T(\lambda + L)^2 (M^2+1) G^2 }{2\lambda L (\lambda + L)^2-\lambda^2L^2}.
$$
Since $L_{\mathcal{D}}(w)$ is $\sm$-smooth, we have 
$$
\underset{\mathcal{A},X \sim \mathcal{D}^n}{\mathbb{E}} \left[L_{\mathcal{D}}(w^T)\right] - L_{\mathcal{D}}(w^*) \leq \frac{\sm}{2} \mathbb{E} [ \| w^T -w^* \|^2_2] \leq  \frac{T(\lambda + L)^2 (M^2+1) G^2 }{4\lambda  (\lambda + L)^2-2\lambda^2L}.
$$ 
which concludes the proof.


\subsection{Proof of Theorem~\ref{thm:main-variation}}
\label{app:proof_oracle}
\begin{algorithm}[H]
\caption{High-Dimensional Mean Estimator}
\label{alg:cdphdme}
\begin{algorithmic}[1]
\State {\bfseries Input:} Samples ${X = \{x_i\}_{i=1}^{n}, x_i \in \mathbb{R}^d}$. Parameters $0<R<10, \tau \ge 10$
    \State Set parameters: $m \leftarrow 4\log(2d/\beta)$
    	\State  $I = [-3\tau, 3\tau]$
    	\For{$j \leftarrow 1,...,d$}
    	\For{$i \leftarrow 1,...,m$}
		 \State $Z_j^i \leftarrow \Paren{ \text{clip} (x, I )~\text{for}~x\in  \Paren{X_{(i-1)\cdot \frac{n}{m}+1}(j),\ldots, X_{i\cdot \frac{n}{m}}(j) } }$
		 \State $\hat{\mu}_j^i \leftarrow \frac{m}{n} \sum_{x \in  \textcolor{black}{Z_j^i}} x$
		 \EndFor
		 \State $\hat{\mu}_j = \text{median} \Paren{\hat{\mu}_j^1,\ldots, \hat{\mu}_j^m}$
	\EndFor
	\State Let $\hat{\mu} = \Paren{\hat{\mu}_1, \ldots, \hat{\mu}_d}$
	\State {\bfseries Output:} $\hat{\mu}$
\end{algorithmic}
\end{algorithm}




%

\newhz{We adopt a coordinate analysis for the algorithm. For each coordinate,}
Algorithm~\ref{alg:cdphdme} truncates each point to an interval. We first recall a lemma from \cite{KamathSU20}, which quantifies the bias induced.

\begin{lemma}\label{lem:1-d-truncation bias}[Lemma 3.1 in~ \cite{KamathSU20}]
\label{lemma_KSU_31}
    Let $\tau \ge 10$, and $\cD$ be a distribution over $\R$ with mean
    $\mu$ and $k$-th moment bounded by $1$.
    Suppose $x \sim \cD$, $\ce \in \mathbb{R}$ and $Z$ is the following random variable,
    \[
    Z = \left\{
    \begin{aligned}
    a &~\text{if}~ x < \ce - 3\tau, \\
    b &~\text{if}~ x >  \ce + 3\tau, \\
    x &~\text{if}~ x \in  [\ce - 3\tau, \ce + 3\tau]. \\
    \end{aligned}
    \right.
    \]
If $\mu - \ce  \le 3\tau$, then $\absv{\mu - \expectation{Z}} \le 3\cdot \Paren{\frac{C}{\tau}}^{k-1},$ where \textcolor{black}{$C \geq 14$ is a universal constant.}  
    \end{lemma}

Intuitively, this lemma tells that if the heavy-tailed random variable is truncated to an interval with length $6\tau$ and its center $\ce$ close enough to the true mean $\mu$, the induced bias is small. With this in mind, we proceed to prove Theorem \ref{thm:main-variation}.
\begin{proof}[Proof of Theorem \ref{thm:main-variation}]


We firstly show the accuracy guarantees of the non-private algorithms. 

We analyze the algorithm coordinatewisely. For a fixed dimension $j$, let $Z_j = \text{clip} (x(j), I )$ with $x \sim \cD$, we note that the $k$-th moment of $Z_j$ is  bounded by 1. Since $R \le \tau$, by Lemma~\ref{lem:1-d-truncation bias}, 
\begin{align}
\label{equ:1-d-truncation bias}
\absv{\expectation{Z_j} - \mu_j} \le 3\cdot \Paren{\frac{C}{\tau}}^{k-1},
\end{align} 
where $\mu_j = \expectation{X(j)}$.

Let $m = 4\log(2d/\beta)$. For a fixed $i$, $Z^i_j$ is a combination of $\frac{n}{m}$ i.i.d. realizations of $Z_j$. By Lemma~\ref{lem:mean_concentration}, we have 
\[
\probof{\absv{\hat{\mu}^i_j - \expectation{Z_j}} \le 10\cdot \sqrt{\frac{m}{n}}} \ge 0.9.
\]

Note that  $\hat{\mu}_j = \text{median} \Paren{\hat{\mu}_j^1,\ldots, \hat{\mu}_j^m}$. By Hoeffding's inequality,
\[
\probof{\absv{\hat{\mu}_j - \expectation{Z_j}} \ge 10\cdot \sqrt{\frac{m}{n}}} \le e^{-\frac{m}{4}}.
\]

We apply the union bound to all the dimensions. Combined with~\eqref{equ:1-d-truncation bias}, we get
\[
\probof{\norm{\hat{\mu} - \mu} \ge \sqrt{d} \cdot \Paren{10 \cdot \sqrt{\frac{m}{n}} + 3\cdot \Paren{\frac{C}{\tau}}^{k-1}}} \le d\cdot e^{-\frac{m}{4}} \le \frac{\beta}{2}.
\]

Next we move to private adaptions, where the key step is to bound the sensitivity of the non-private algorithm.

Fixing one dimension $j \in [d]$, for two neighboring datasets $X$ and $X^\prime$, we want to show that $\absv{\hat{\mu}_j (X) - \hat{\mu}_j (X^\prime)} \le \frac{12\tau m}{n}$ for each $j$.
With this in mind, $\ell_1$ sensitivity of $\hat{\mu}$ is upper bounded by $\frac{12\tau m d}{n}$, and the $\ell_2$ sensitivity is upper bounded by $\frac{12\tau m \sqrt{d}}{n}$.

\gk{
Now it suffices to bound the $\ell_{\infty}$ sensitivity.
Let $\hat{\mu}_j = median(\hat{\mu}_j^1,\ldots,\hat{\mu}_j^i,\ldots, \hat{\mu}_j^m)$. Let $X$ and $X^\prime$ be the two neighboring datasets which differ at one sample.
Suppose $m$ is odd, there are two cases: 
\begin{enumerate}
\item $\hat{\mu}_j^{i^*}(X)$ is the median for dataset $X$, and $\hat{\mu}_j^{i^*}(X^\prime)$ is the median for dataset $X^\prime$.
\item $\hat{\mu}_j^{i^*}(X)$ is the median for dataset $X$, while $\hat{\mu}_j^{i^\prime}(X^\prime)$ is the median for dataset $X^\prime$.
\end{enumerate}
For the first case, $\absv{\hat{\mu}_j^{i^*}(X)- \hat{\mu}_j^{i^*}(X^\prime)} \le \frac{12\tau m}{n}$, since $X$ and $X^\prime$ differ at one sample. For the second case, note that it can only happen when $\absv{\hat{\mu}_j^{i^*}(X)- \hat{\mu}_j^{i^\prime}(X)} \le \frac{6\tau m}{n}$. Furthermore, we have $\absv{\hat{\mu}_j^{i^\prime}(X)- \hat{\mu}_j^{i^\prime}(X^\prime)} \le \frac{6\tau m}{n}$. By triangle inequality, $\absv{\hat{\mu}_j^{i^*}(X)- \hat{\mu}_j^{i^\prime}(X^\prime)} \le \frac{12\tau m}{n}$, which provides an upper bound of the $\ell_{\infty}$ sensitivity.
The case when $m$ is even is similar and omitted.\footnote{To facilitate the sensitivity analysis, let $\{x'_i\}_{i=1}^\ns$ be the ordered set of $\{x_i\}_{i=1}^\ns$. If $n$ is even, we define the median to be $\frac12\Paren{x'_{\frac{n}{2}}+x'_{\frac{n}{2}+1 } }$ rather than an arbitrary value ranging from $x'_{\frac{n}{2}}$ to $x'_{\frac{n}{2}+1}$.}
}

For CDP adaption, by the guarantee of Gaussian mechanism (Lemma~\ref{lem:additive_noise}), the algorithm satisfies $\rho$-CDP when the noise added is $\normal\Paren{0, \frac{\textcolor{black}{72}\tau^2 m^2 d}{\rho n^2} \mathbb{I}_{d \times d} }$.

Besides, since $N \sim \normal \Paren{0, \sigma^2 \mathbb{I}_{d \times d}}$, where $\sigma^2 = \frac{\textcolor{black}{72}\tau^2 m^2 d}{\rho n^2}$. By the tail property of chi-squared distribution~\cite{LaurentM20},
\[
\probof{\norm{N} \ge 2\sigma\Paren{\sqrt{d}+\sqrt{\log\Paren{\frac1{\beta}}} }} \le \frac{\beta}{2}.
\]

Note that $\norm{ \widetilde{\mu} - \mu} \le \norm{N} +   \norm{\hat{\mu} - \mu}$, we conclude the proof by the union bound.

With respect to DP adaption, by the guarantee of Laplace mechanism  (Lemma~\ref{lem:additive_noise}), the algorithm satisfies $\eps$-DP when the noise added is  $\text{Lap}\Paren{0,\frac{12\tau md}{\eps n}}$ for each dimension.

Besides, let $N_j \sim \text{Lap}\Paren{0,\frac{12\tau md}{\eps n}}$, by the tail property of Laplace distribution, 
\[ 
\probof{\absv{N_j} \ge \frac{48\tau d}{\eps n} \cdot \log^2(2d/\beta)} \le \frac{\beta}{2d}.
\]
By union bound, with probability at least $1-\frac{\beta}{2}$,
\[ 
\probof{\norm{N} \ge \frac{48\tau d^{\frac{3}{2}}}{\eps n} \cdot \log^2(2d/\beta)} \le \frac{\beta}{2}.
\]
Note that $\norm{ \widetilde{\mu} - \mu} \le \norm{N} +   \norm{\hat{\mu} - \mu}$, we conclude the proof by applying the union bound.

\end{proof}

\subsection{Proof of Theorem \ref{cor:corollary15}}
\label{add2ub2}

\begin{lemma} 
  Consider Algorithm~\ref{alg:reduce} instantiated with CDPCWME$\left(\frac{\rho}{{T}}, \tau \right)$ as MeanOracle (Algorithm~\ref{alg:cdphdme}). Under Assumption \ref{assn:main-assumption} and further assuming $R \le 10$, $\sm \le 10$, when $\tau \ge 10$, the following holds for all $ w \in \mathcal{W}$ simultaneously:
\begin{align*}
\| \mathbb{E}[\nabla\widetilde{L}_{\mathcal{D}}(w)] - \nabla L_{\mathcal{D}}(w) \|_2 \le \widetilde{O}\Paren{\frac{d}{\sqrt{n}}+ \sqrt{d} \cdot \Paren{\frac{C}{\tau}}^{k-1}}.
\end{align*}
and
\begin{align*}
\mathbb{E}[\| \nabla\widetilde{L}_{\mathcal{D}}(w) - \nabla L_{\mathcal{D}}(w) \|_2^2] \leq   \widetilde{O}\Paren{ \frac{\tau^2 d^4 T}{\rho n^2}+\frac{d^2}{n}+ d \cdot \Paren{\frac{C}{\tau}}^{2k-2}}.
\end{align*}
  where $ \nabla\widetilde{L}_{\mathcal{D}}(w)$ is the estimated gradient in Algorithm~\ref{alg:reduce}.
\end{lemma} 

\begin{proof}

We start with bounding the bias. First we note that $\mathbb{E}[\nabla\widetilde{L}_{\mathcal{D}}(w)|\nabla\hat{L}_{\mathcal{D}}(w)] = \nabla\hat{L}_{\mathcal{D}}(w)$, which denotes the output of the non-private algorithm.

In order to obtain bounds that hold uniformly over the choice of $w$, we follow a standard strategy of covering. Note that the number of balls of radius $\dist$ required to cover $\cW$ is bounded as $N_{\dist} \le \Paren{\frac{M}{\dist}}^d$.
Let $\mathcal{W}_{\dist} = \{ \widetilde{w}_1, \ldots, \widetilde{w}_{N_{\dist}} \}$ denote the centers of this covering. For an arbitrary $w \in W$, and $\widetilde{w} \in \cW_\dist$,
\[
\norm{\nabla\hat{L}_{\mathcal{D}}(w) - \nabla L_{\mathcal{D}}(w)} \le \norm{\nabla\hat{L}_{\mathcal{D}}(w)  - \nabla\hat{L}_{\mathcal{D}}(\widetilde{w}) } +\norm{\nabla\hat{L}_{\mathcal{D}}(\widetilde{w})  - \nabla L_{\mathcal{D}}(\widetilde{w}) } +\norm{\nabla L_{\mathcal{D}}(\widetilde{w})  - \nabla L_{\mathcal{D}}(w)}.
\]

We bound each term separately.

For the first term, we need to analyze how much the output of the non-private estimator $\nabla\hat{L}_{\mathcal{D}}(\cdot)$ changes, when the input switches from $w$ to $\widetilde{w}$. 

Let $\beta = \Paren{\frac{\alpha}{M}}^{2d}$, and $m = 4\log(2d/\beta) = 8d \log\Paren{\frac{3M}{\alpha}}$. According to the smoothness assumption, for each dimension $j$, and batch $i \in [m]$,  the average of each batch differs by no more than  $\sm\dist$. Therefore, for each dimension $j$, the median differs by no more than $\sm m\dist$. Summing over all the dimensions,

\begin{align*}
\norm{\nabla\hat{L}_{\mathcal{D}}(w)  - \nabla\hat{L}_{\mathcal{D}}(\widetilde{w})} \le \sm m\dist \cdot \sqrt{d}. 
\end{align*}

For the second term, let  $\beta = \Paren{\frac{\alpha}{M}}^{2d}$. According to Theorem~\ref{thm:main-variation}, with probability at least $1-\beta$,
\begin{align*}
\norm{\nabla\hat{L}_{\mathcal{D}}(\widetilde{w})  - \nabla L_{\mathcal{D}}(\widetilde{w})} \le  C^\prime\Paren{\sqrt{d} \cdot \Paren{\sqrt{\frac{\log\Paren{\frac{d}{\beta}}}{n}} + \Paren{\frac{C}{\tau}}^{k-1}}},
\end{align*}
where $C^\prime$ is a universal constant.

Note that $\beta \cdot N_\alpha \le \Paren{\frac{\alpha}{M}}^d$. By union bound, with probability at least $1- \Paren{\frac{\alpha}{M}}^d$, for all $\widetilde{w} \in \mathcal{W}_\alpha$,
\begin{align*}
\norm{\nabla\hat{L}_{\mathcal{D}}(\widetilde{w})  - \nabla L_{\mathcal{D}}(\widetilde{w})} \le  C^\prime\Paren{\sqrt{d} \cdot \Paren{\sqrt{\frac{\log\Paren{\frac{d}{\beta}}}{n}} + \Paren{\frac{C}{\tau}}^{k-1}}},
\end{align*}

Note that $\norm{\nabla\hat{L}_{\mathcal{D}}(\widetilde{w})  - \nabla L_{\mathcal{D}}(\widetilde{w})} \le 2R$ for sure. Taking expectation, we have
\begin{align*}
\expectation{\norm{\nabla\hat{L}_{\mathcal{D}}(\widetilde{w})  - \nabla L_{\mathcal{D}}(\widetilde{w})} } \le C^\prime\Paren{\sqrt{d} \cdot \Paren{\sqrt{\frac{\log\Paren{\frac{d}{\beta}}}{n}} + \Paren{\frac{C}{\tau}}^{k-1}}}+2R\cdot \Paren{\frac{\alpha}{M}}^d.
\end{align*}

For the third term, according to the smoothness assumption,
\[
\norm{\nabla L_{\mathcal{D}}(\hat{w})  - \nabla L_{\mathcal{D}}(w)} \le \sm \dist.
\]

Summing up all three terms, we have 
\begin{align*}
\expectation{\norm{\nabla\hat{L}_{\mathcal{D}}(\widetilde{w})  - \nabla L_{\mathcal{D}}(\widetilde{w})}} \leq   \sm \alpha (m\sqrt{d}+1)+C^\prime\Paren{\sqrt{d} \cdot \Paren{\sqrt{\frac{\log\Paren{\frac{d}{\beta}}}{n}} + \Paren{\frac{C}{\tau}}^{k-1}}}+2R\cdot \Paren{\frac{\alpha}{M}}^d.
\end{align*}

Finally taking $\alpha = \frac1{n^3}$, and note that $\mathbb{E}[\nabla\widetilde{L}_{\mathcal{D}}(w)] = \expectation{\mathbb{E}[\nabla\widetilde{L}_{\mathcal{D}}(w)]| \nabla\hat{L}_{\mathcal{D}}(w) } = \mathbb{E}[\nabla\hat{L}_{\mathcal{D}}(w)]  $,
\begin{align}
\label{equ:bias_bound}
\| \mathbb{E}[\nabla\widetilde{L}_{\mathcal{D}}(w)] - \nabla L_{\mathcal{D}}(w) \|_2 &=   \| \mathbb{E}[\nabla\hat{L}_{\mathcal{D}}(w)] - \nabla L_{\mathcal{D}}(w) \|_2 \nonumber\\
&\le \expectation{\norm{\nabla\hat{L}_{\mathcal{D}}(\widetilde{w})  - \nabla L_{\mathcal{D}}(\widetilde{w})}}\nonumber\\
&\le \widetilde{O}\Paren{ \frac{ \sm d^{1.5}}{n^3}+\frac{d}{\sqrt{n}}+ \sqrt{d} \cdot \Paren{\frac{C}{\tau}}^{k-1}+\frac{R}{n^{3d}}}.
\end{align}

Assuming $ \sm $, $R$ are constants,
\[
\| \mathbb{E}[\nabla\widetilde{L}_{\mathcal{D}}(w)] - \nabla L_{\mathcal{D}}(w) \|_2 \le \widetilde{O}\Paren{\frac{d}{\sqrt{n}}+ \sqrt{d} \cdot \Paren{\frac{C}{\tau}}^{k-1}}.
\]

Next we move to the variance. Note that 
\begin{align*}
\mathbb{E}[\| \nabla\widetilde{L}_{\mathcal{D}}(w) - \nabla L_{\mathcal{D}}(w) \|_2^2] \leq  2\mathbb{E}[\| \nabla\widetilde{L}_{\mathcal{D}}(w) - \nabla \hat{L}_{\mathcal{D}}(w) \|_2^2] +2\mathbb{E}[\| \nabla\hat{L}_{\mathcal{D}}(w) - \nabla L_{\mathcal{D}}(w) \|_2^2].
\end{align*}

\textcolor{black}{
As designed in Theorem \ref{thm:main-variation}, we notice that 
\begin{align}
\nabla \hat{L}_{\mathcal{D}}(w^t) -   \nabla \widetilde{L}_{\mathcal{D}}(w^t) = N_t \sim \normal\Paren{0, \sigma \mathbb{I}_{d \times d}}, \nonumber
\end{align}
where $\sigma^2 = \frac{{72}\tau^2 m^2 dT}{\rho n^2}$.}

\textcolor{black}{Thus, $\mathbb{E}[\| \nabla\widetilde{L}_{\mathcal{D}}(w) - \nabla \hat{L}_{\mathcal{D}}(w) \|_2^2] = \frac{72\tau^2m^2 d^2 T}{\rho n^2}$,} and by~\eqref{equ:bias_bound},
\begin{align*}
\mathbb{E}[\| \nabla\widetilde{L}_{\mathcal{D}}(w) - \nabla L_{\mathcal{D}}(w) \|_2^2] \leq   \widetilde{O}\Paren{ \frac{\tau^2 d^4 T}{\rho n^2}+ \frac{ \sm^2 d^{3}}{n^6}+\frac{d^2}{n}+ d \cdot \Paren{\frac{C}{\tau}}^{2k-2}+\Paren{\frac{R}{n^{3d}}}^{2}}.
\end{align*}

Assuming $\sm$, $R$ are constants,
\begin{align*}
\mathbb{E}[\| \nabla\widetilde{L}_{\mathcal{D}}(w) - \nabla L_{\mathcal{D}}(w) \|_2^2] \leq   \widetilde{O}\Paren{ \frac{\tau^2 d^4 T}{\rho n^2}+\frac{d^2}{n}+ d \cdot \Paren{\frac{C}{\tau}}^{2k-2}}.
\end{align*}

\end{proof}

The proof now follows by choosing the right $\eta$ and $T$ in Theorem~\ref{thm:sco-f-c}. To balance the first two terms, $\frac{M^2}{\eta T} + \frac{\eta}{2}R^2$, in Theorem~\ref{thm:sco-f-c}, we let $\eta = \frac{M}{R \sqrt{T}}$. 

Suppose $T = \frac{R^2\rho n^2}{\tau^2 d^4}$, with $\tau = \Paren{\frac{\sqrt{\rho} n}{Md^{\frac{3}{2}}}}^{\frac1k}$
we have

\begin{align*}
\frac{M^2}{\eta T} = \frac{\eta}{2}R^2 = O\Paren{ \sqrt{d}\cdot \Paren{\frac{Md^{\frac{3}{2} }}{\sqrt{\rho}n}}^{\frac{k-1}{k}}  } .
\end{align*}
Besides,
\begin{align*}
BM &= \widetilde{O}\Paren{\frac{Md}{\sqrt{n}}+ M\sqrt{d}\cdot \Paren{\frac{Md^{\frac{3}{2} }}{\sqrt{\rho}n}}^{\frac{k-1}{k}} }.
\end{align*}
Finally,
\begin{align*}
\eta G^2  = \widetilde{O}\Paren{ \frac{d^2}{n}+  d\cdot \Paren{\frac{Md^{\frac{3}{2} }}{\sqrt{\rho}n}}^{\frac{2k-2}{k}}+ M\sqrt{d}\cdot \Paren{\frac{Md^{\frac{3}{2} }}{\sqrt{\rho}n}}^{\frac{k-1}{k}} }.
\end{align*}
Putting the various terms together completes the proof.

\subsection{\newhz{Proof of Theorem \ref{thm:ub_sco_k_2}}}
\label{appendix:proof_k_2}
We first introduced the mean estimation oracle in~\cite{Holland19}. For $x \in \mathbb{R}$, let 
\[ \phi(x) = \left \{
\begin{array}{ccl}
x - \frac{x^3}{6},&  & -\sqrt{2} \le x \le \sqrt{2}, \nonumber\\
\frac{2\sqrt{2}}{3},&  &  x > \sqrt{2},\nonumber\\
-\frac{2\sqrt{2}}{3},& & x < -\sqrt{2}.\nonumber
\end{array}
\right.
\]

\begin{algorithm}[H]
\caption{CDP Noise Smoothing Mean Estimator~\cite{Holland19, WangXDX20}}
\label{alg:holland}
\begin{algorithmic}[1]
\State {\bfseries Input:} Samples ${X = \{x_i\}_{i=1}^{n}, x_i \in \mathbb{R}^d}$. Parameters $\rho, \tau \ge 10$
	\State Let $\np = \normal\Paren{0,c}$, and $N \sim p$, where $c$ is a constant
    	\For{$j \leftarrow 1,...,d$}
		\State $\hat{\mu}_j = \frac{\tau}{n} \sum_{i=1}^{\ns} \int_{-\infty}^{\infty} \phi\Paren{\frac{x_i(j)(1+N)}{\tau}}d\np(N)$
		 \State Let $\hat{\mu} = \Paren{\hat{\mu}_1, \ldots, \hat{\mu}_d}$
	\EndFor
	 \State {\bfseries Output:} $\hat{\mu}+\normal\Paren{0,  \frac{\tau^2 d}{\rho n^2} \cdot \mathbb{I}_{d \times d}}$
\end{algorithmic}
\end{algorithm}

\begin{remark}
This estimator can be efficiently computed. Please refer to~\cite{Holland19, WangXDX20} for more detail.
\end{remark}

It is not hard to see this algorithm satisfies $\rho$-CDP. We note that $\forall x, \absv{\phi(x)} \le 1$, so the $\ell_2$ sensitivity $\Delta_2(\hat{\mu}) \le \sqrt{d}$. We conclude the proof by applying Lemma~\ref{lem:additive_noise}.

We provide the accuracy guarantee of this algorithm in the following lemma.

\begin{lemma} 
  Consider Algorithm~\ref{alg:reduce} instantiated with \newhz{CDPNSME$\left(\frac{\rho}{{T}}, \tau \right)$ as MeanOracle (Algorithm~\ref{alg:holland})}. Under Assumption \ref{assn:main-assumption} and further assuming $R \le 10$, $\sm \le 10$, when $\tau \ge 10$, the following holds for all $ w \in \mathcal{W}$ simulatenously:
\begin{align*}
\| \mathbb{E}[\nabla\widetilde{L}_{\mathcal{D}}(w)] - \nabla L_{\mathcal{D}}(w) \|_2 \le \widetilde{O}\Paren{\frac{d^{\frac{3}{2}} \tau }{n}+\frac{\sqrt{d}}{\tau}}.
\end{align*}
and
\begin{align*}
\mathbb{E}[\| \nabla\widetilde{L}_{\mathcal{D}}(w) - \nabla L_{\mathcal{D}}(w) \|_2^2] \leq   \widetilde{O}\Paren{\frac{d^3 \tau^2}{n^2} + \frac{d}{\tau^2}+ \frac{\tau^2d^2T}{\rho n^2}}.
\end{align*}
  where $ \nabla\widetilde{L}_{\mathcal{D}}(w)$ is the estimated gradient in Algorithm~\ref{alg:reduce}.
\end{lemma} 

\begin{proof}
This bias analysis directly comes from combining Remark 3 and Lemma 4 in~\cite{Holland19}.
In fact, this analysis can be viewed as Lemma 5 of~\cite{Holland19} with an explicit analysis on $\tau$. \footnote{One may wonder why our result is different with Lemma 5 in~\cite{Holland19} when setting $\tau = \sqrt{\frac{n}{d}}$. After communicating with authors of~\cite{Holland19}, we confirmed there was an issue in their Lemma 5, where their $s_j$ (equivalent with our $\tau$) should be $\sqrt{\frac{n}{d}}$ instead of $\sqrt{n}$.}

With respect to the variance analysis,
\begin{align*}
\mathbb{E}[\| \nabla\widetilde{L}_{\mathcal{D}}(w) - \nabla L_{\mathcal{D}}(w) \|_2^2] \leq  2\mathbb{E}[\| \nabla\widetilde{L}_{\mathcal{D}}(w) - \nabla \hat{L}_{\mathcal{D}}(w) \|_2^2] +2\mathbb{E}[\| \nabla\hat{L}_{\mathcal{D}}(w) - \nabla L_{\mathcal{D}}(w) \|_2^2].
\end{align*}
Note that the noise added is generated from $N \sim \normal\Paren{0, \frac{\tau^2 dT}{\rho n^2} \mathbb{I}_{d \times d}}$. We conclude the proof by summing over all the dimensions.

\end{proof}

The proof now follows by choosing the right $\eta$ and $T$ in Theorem~\ref{thm:sco-f-c}. To balance the first two terms, $\frac{M^2}{\eta T} + \frac{\eta}{2}R^2$, in Theorem~\ref{thm:sco-f-c}, we let $\eta = \frac{M}{R \sqrt{T}}$. 

Suppose $T = \frac{R^2\rho n^2}{\tau^2 d^2}$, with $\tau = \Paren{\frac{\sqrt{\rho} n}{Md^{\bal}}}^{\frac1{2}}$
we have

\begin{align*}
\frac{M^2}{\eta T} = \frac{\eta}{2}R^2 = O\Paren{\frac{\sqrt{M} d^{1-\frac{q}{2} }} {\rho^{\frac14}\sqrt{n}} } .
\end{align*}
Besides,
\begin{align*}
BM &= \widetilde{O}\Paren{\frac{Md^{\frac{3}{2}} }{n} \cdot  \Paren{\frac{\sqrt{\rho} n}{Md^{ \bal }}}^{\frac1{2}} + M\sqrt{d}\cdot  \Paren{\frac{\sqrt{\rho} n}{Md^{\bal }}}^{-\frac1{2}} }= \widetilde{O}\Paren{ \frac{\sqrt{M}d^{\frac{3-\bal}{2}}}{\sqrt{n}}+\frac{\sqrt{M} d^{\frac{1+\bal}{2} }} {\rho^{\frac14}\sqrt{n}}  }.
\end{align*}
Finally,
\begin{align*}
\eta G^2  = \widetilde{O}\Paren{ \frac{d^{3-\bal }}{n}+  \frac{Md^{1+\bal }}{\sqrt{\rho}n}+  \frac{\sqrt{M} d^{1-\frac{\bal}{2} }} {\rho^{\frac14}\sqrt{n}} }.
\end{align*}
Note that when $0.5\le \bal \le 2$, $\frac{1+\bal}{2} \ge 1 - \frac{\bal}{2}$, putting the various terms together completes the proof.


\subsection{Proof of Theorem \ref{cor:corollary14}}
\label{add2ub1}

In the strongly convex setting, for each iteration, the input of the MeanOracle is disjoint and independent, with size $\frac{n}{T}$. Therefore, there is no need to adopt the strategy of covering.

The following Lemma is a direct consequence  of Theorem \ref{thm:main-variation} by setting $\beta = {\frac{Td}{ n}} + {d} \cdot \left(\frac{\sqrt{d}T^{\frac{3}{2}}}{\sqrt{\rho} n}\right)^{\frac{2k-2}{k}}$ for \xtl{CDPCWME$\left(\frac{\rho}{{T}}, {\left(\frac{\sqrt{\rho} n}{\sqrt{d}T^{\frac{3}{2}}}\right)}^{1/k}\right)$}.

\begin{lemma} 
  Consider Algorithm~\ref{alg:reduce} instantiated with \xtl{CDPCWME$\left(\frac{\rho}{{T}}, {\left(\frac{\sqrt{\rho} n}{\sqrt{d}T^{\frac{3}{2}}}\right)}^{1/k}\right)$} as MeanOracle.
  Under Assumption \ref{assn:main-assumption}, the following holds for all $ w^t$, $t \in [T]$:
\begin{align*}
\mathbb{E} [ \| \nabla\widetilde{L}_{\mathcal{D}}(w^t) - \nabla L_{\mathcal{D}}(w^t) \|_2^2] \leq   \widetilde{O}\left({\frac{Td}{ n}} + {d} \cdot \left(\frac{\sqrt{d}T^{\frac{3}{2}}}{\sqrt{\rho} n}\right)^{\frac{2k-2}{k}} \right),
\end{align*}
  where $ \nabla\widetilde{L}_{\mathcal{D}}(\cdot)$ is the estimated gradient in Algorithm~\ref{alg:reduce}.
\end{lemma}

Note that $T$ is poly-logarithmic on $n$ and $d$. The proof follows by Theorem~\ref{thm:sco-f-sc} immediately.


\subsection{Proof of Lemma \ref{lem:reduction_c}}
\label{lb1}

Let $x \sim \mathcal{D}$, and $\ell(w;x) = \frac12\norm{w-x}^2$.
Note that $w^* = \arg\min L_{\mathcal{D}}(w) = \underset{x\sim \mathcal{D}}{\mathbb{E}}[x] = \mu $.  
Further using the expansion $\|a-b\|^2_2 = \norm{a}^2 - 2\langle a,b \rangle+\norm{b}^2$,
\begin{align*}
  L_{\mathcal{D}}(w) - L_{\mathcal{D}}(w^*) &= \frac12\underset{x\sim \mathcal{D}}{\mathbb{E}}[\norm{w-x}^2 - \norm{w^*-x}^2] \\ 
  &= \frac12\underset{x\sim \mathcal{D}}{\mathbb{E}}\left[\|w\|_2^2 - 2 \langle w, x \rangle + \|x\|_2^2 - \|w^*\|_2^2 + 2 \langle w^*, x \rangle - \|x\|_2^2\right] \\
  &= \frac12\left(\|w\|_2^2 - 2 \langle w, w^* \rangle  - \|w^*\|_2^2 + 2 \langle w^*, w^* \rangle \right) \\
  &= \frac12\left(\|w\|_2^2 - 2 \langle w, w^* \rangle  + \|w^*\|_2^2  \right) \\
  &= \frac12\|w - w^*\|_2^2 
\end{align*}

Notice that $\ell$ is both strongly convex and smooth and the expected risk of $w^{priv}$ is 
$$
\underset{X\sim \mathcal{D}^n, \mathcal{A}}{\mathbb{E}}[L_{\mathcal{D}}(w^{priv}) ] - L_{\mathcal{D}}(w^{*}) = \underset{X\sim \mathcal{D}^n, \mathcal{A}}{\mathbb{E}}\left[ \frac12\| w^{priv} - \mu \|^2_2\right],
$$
which implies the result.

\hz{
Now we prove the second half, note that $\nabla \ell (w,x) = w-x$, and $\expect{\nabla \ell (w,x)} = w-\mu$,
\begin{align*}
&~~~\sup_{j \in [d]}\expectsub{ \absv{ \langle e_j, \nabla \ell(w,x)-\expect{\nabla \ell(w,x)} \rangle}^k } {x \sim \mathcal{D}}\\
&= \sup_{j \in [d]}\expectsub{ \absv{ \langle e_j, w-x - (w-\mu) \rangle}^k } {x \sim \mathcal{D}}\\
&= \sup_{j \in [d]}\expectsub{ \absv{ \langle e_j, x - \mu \rangle}^k } {x \sim \mathcal{D}} \le 1.
\end{align*}
}
 

\subsection{Proof of Lemma \ref{lem:duchi_res}}
\label{sec:proof_lb_me}

We first prove the private term (the second term) in Lemma \ref{lem:duchi_res} for $(\eps,0)$-DP.

We adopt the packing set defined in the proof of Proposition 4 in~\cite{BarberD14}. 
Given $\nu \in \cV$, with $\normo{\nu}=\frac{d}{2}$, and \newhz{$ \nu \in \{\pm 1\}^d$}, let $Q_{\nu} = (1-p)P_0 + p P_{\nu}$ for some $p \in [0,1]$, where $P_0$ is a point mass on $\{ D = 0\}$ and $P_{\nu}$ is a point mass on $\{ D =  p^{-1/k}\nu \}$.

Given $Q_\nu$, we define $\mu_\nu \in \mathbb{R}^d$ to be the mean of $Q_\nu$, i.e., $\mu_\nu = \mathbb{E}_{x\sim Q_{\nu}}[x]$. 

\newhz{
As a corollary of standard Gilbert-Varshamov bound for constant-weight codes (e.g., see Lemma 6 in~\cite{AcharyaSZ21}), there exists a set $\cV$ such that
\begin{itemize}
\item The cardinality of $\cV$ satisfies $\absv{\cV} \ge  2^{\frac{d}{8}}$.
\item For all $\nu \in \cV$, $\nu \in \{\pm 1\}^d$ with $\normo{\nu} = \frac{d}{2}$.
\item For all $\nu_1, \nu_2 \in \cV$, $\dham\Paren{\nu_1, \nu_2} \ge \frac{d}{8}$.
\end{itemize}
}

We first compute the norm of $\mu_\nu$. Note that $\forall \nu \in \cV$, $\norm{\mu_\nu}$ is the same, which is denoted by $\norm{\mu}$.
$$
\| \mu_\nu \|_2 = \| \mathbb{E}_{x\sim Q_{\nu}}[x] \|_2 = p^{\frac{k-1}{k}} \newhz{\cdot \sqrt{\frac{d}{2}}} \coloneqq \norm{\mu}.
$$

\newhz{
Let $x\sim Q_\nu$, and $e_j$ denote the $j$-th standard basis.
\begin{align*}
\sup_{j \in [d]} \expectsub{ \absv{ \langle \Paren{x - \mu_{\nu}}, e_j \rangle }^k } {x \sim Q_\nu} \le p \cdot \Paren{p^{-1/k}}^k = 1.
\end{align*}
}

Now we are able to bound the error.
\begin{align*}
\underset{X\sim \mathcal{D}^n, \mathcal{A}}{\mathbb{E}} [\| \mathcal{A}(X) - \mu \|_2] & \ge \frac1{\absv{\cV}}  \sum_{\nu \in \cV}  \expectsub {\| \mu^{priv}(X) - \mu_\nu  \|_2}{X \sim  Q_\nu^{\ns}},
\end{align*}
which comes from the fact that the worst case loss is no smaller than the average loss.

Note that $\absv{\cV} \ge 2^{\frac{d}{8}}$. Furthermore, $\forall \nu \neq \nu^\prime$, $\norm{\mu_\nu - \mu_{\nu^\prime}} \ge \frac12 \norm{\mu}$; $\dtv\Paren{Q_\nu, Q_{\nu^\prime}} = p$, indicating that there exists a coupling between $Q_\nu$ and $Q_{\nu^\prime}$ with a coupling distance $np$. Suppose $p = \min\Paren{1, \frac{d}{n\epsilon}}$, by DP Fano's inequality (Theorem 2 in~ \cite{AcharyaSZ21}), it can be shown that

\[
 \frac1{\absv{\cV}}  \sum_{\nu \in \cV}  \expectsub {\| \mu^{priv}(X) - \mu_\nu  \|_2}{X \sim  Q_\nu^{\ns}} =  {\Omega}\Paren{ \min\Paren{1, \left( \frac{d}{\epsilon n}\right)^{\frac{k-1}{k}}} \newhz{\cdot \sqrt{d}} }.
\]

With respect to $\rho$-CDP algorithms, we just take $p = \min \Paren{1, \frac{\sqrt{d}}{\ns\sqrt{\rho}}}$, by CDP Fano's inequality (Theorem~\ref{thm:dp_fano}), it can be shown that
\[
 \frac1{\absv{\cV}}  \sum_{\nu \in \cV}  \expectsub {\| \mu^{priv}(X) - \mu_\nu  \|_2}{X \sim  Q_\nu^{\ns}} =  \xl{\Omega \Paren{ \min\Paren{1, \left(\frac{\sqrt{d}}{\sqrt{\rho} n}\right)^{\frac{k-1}{k}}} \cdot \sqrt{d} }}.
\]
 
We conclude the proof by noting that the non-private term (the first term) in Lemma \ref{lem:duchi_res} comes from classical Gaussian mean estimation, and $\forall a, b_1, b_2$, $a \ge 0.5(b_1+b_2)$ if $a\ge\max(b_1, b_2)$.

\begin{remark}
The previous analysis implicitly assumes the strongly convex parameter $\lambda = 1$. To see the dependency on $\lambda$, we let the loss function $\ell(w;x) = \frac{\lambda}{2}\norm{w-x}^2$ instead. Meanwhile, to keep the $k$-th moment bounded by 1, we have to shrink the parameter space of $x$ by $\lambda$. Therefore, the $\| w^{priv} - \mu \|_2$ gets scaled by $\frac1{\lambda}$, and the final loss gets scaled by $\lambda \cdot \frac1{\lambda^2} = \frac1{\lambda}$. We note that this dependency matches with our upper bound when $\lambda = L$, which is the smoothness parameter.
\end{remark}



\subsection{Proof of Theorem \ref{thm:lower-bound_c}}
\label{lb2}

\begin{theorem}[Convex case] 
Let $n,d \in \mathbb{N}$. 
  There exists a convex and smooth loss function $\ell: \mathcal{W} \times \mathbb{R}^d$, such that for every $(\epsilon,0)$-DP algorithm (whose output on input $X$ is denoted by $w^{priv} = \mathcal{A}(X)$), 
 \hz{there exists a distribution $\mathcal{D}$ on $\mathbb{R}^d$ with $\forall w$, $\sup_{j \in [d]}  \expectsub{ \absv{ \langle \nabla \ell(w,x)-\expect{\nabla \ell(w,x)}, e_j \rangle}^k } {x \sim \mathcal{D}} \le 1$ ($e_j$ is the $j$-th standard basis), which satisfies}
\[
  \underset{X\sim \mathcal{D}^n, \mathcal{A}}{\mathbb{E}}[  L_{\mathcal{D}}(w^{priv}) - L_{\mathcal{D}}(w^*)) ]  \geq \sqrt{ \frac{d}{n} } + \xl{\Omega \Paren{ \min\Paren{1, \left(\frac{d}{\epsilon n}\right)^{\frac{k-1}{k}}} \cdot \sqrt{d} },}
\]
where $w^* = \arg\min_w L_{\mathcal{D}}(w)$.

\newhz{With respect to $\rho$-CDP algorithms, the lower bound turns to
\[
  \underset{X\sim \mathcal{D}^n, \mathcal{A}}{\mathbb{E}}[  L_{\mathcal{D}}(w^{priv}) - L_{\mathcal{D}}(w^*)) ]  \geq \sqrt{ \frac{d}{n} } + \xl{\Omega \Paren{ \min\Paren{1, \left(\frac{\sqrt{d}}{\sqrt{\rho} n}\right)^{\frac{k-1}{k}}} \cdot \sqrt{d} },}
\]
}
\end{theorem}

We first prove the private term (the second term) in Theorem \ref{thm:lower-bound_c}.

Similarly, we adopt the packing set defined in the proof of Proposition 4 in~\cite{BarberD14}. 
Given $\nu \in \cV$, with $\normo{\nu}=\frac{d}{2}$, and \newhz{$ \nu \in \{\pm 1\}^d$}, let $Q_{\nu} = (1-p)P_0 + p P_{\nu}$ for some $p \in [0,1]$, where $P_0$ is a point mass on $\{ D = 0\}$ and $P_{\nu}$ is a point mass on $\{ D =  p^{-1/k}\nu \}$.

Given $Q_\nu$, we define $\mu_\nu \in \mathbb{R}^d$ to be the mean of $Q_\nu$, i.e., $\mu_\nu = \mathbb{E}_{x\sim Q_{\nu}}[x]$. Additionally, we define $w_\nu$ to be its normalization, i.e., $w_\nu = \frac{\mu_\nu}{\norm{\mu_\nu}}$. Note that $w_\nu$ is in the same direction as $\mu_\nu$, with $\norm{w_\nu}=1$.

%
\newhz{
As a corollary of standard Gilbert-Varshamov bound for constant-weight codes (e.g., see Lemma 6 in~\cite{AcharyaSZ21}), there exists a set $\cV$ such that
\begin{itemize}
\item The cardinality of $\cV$ satisfies $\absv{\cV} \ge  2^{\frac{d}{8}}$.
\item For all $\nu \in \cV$, $\nu \in \{\pm 1\}^d$ with $\absv{\nu} = \frac{d}{2}$.
\item For all $\nu_1, \nu_2 \in \cV$, $\dham\Paren{\nu_1, \nu_2} \ge \frac{d}{8}$.
\end{itemize}
}


We first compute the norm of $\mu_\nu$. Note that $\forall \nu \in \cV$, $\norm{\mu_\nu}$ is the same, which is denoted by $\norm{\mu}$.
$$
\| \mu_\nu \|_2 = \| \mathbb{E}_{x\sim Q_{\nu}}[x] \|_2 =p^{\frac{k-1}{k}}\newhz{\cdot \sqrt{\frac{d}{2}}} \coloneqq \norm{\mu}.
$$

Without loss of generality, we assume the parameter space $\norm{W} = 1$, which is a unit ball. Then we define the loss function $\ell(w;x)$. Given $\nu \in \cV$, and $x \sim Q_\nu$, we let
$$
\ell(w;x) = -\langle w, x \rangle,
$$
and
$$
L_{Q_\nu}(w) = \underset{x\sim Q_\nu}{\mathbb{E}}[\ell(w;x)] = -\langle w, \mu_\nu \rangle.
$$


\newhz{
Note that $\ell$ is both convex and smooth. Let $x\sim Q_\nu$. Note that $\nabla \ell (w,x) = -x$, and $\expect{\nabla \ell (w,x)} = -\mu_\nu$,
\begin{align*}
&~~~\sup_{j \in [d]}\expectsub{ \absv{ \nabla_j \ell(w,x)-\expect{\nabla_j \ell(w,x)}}^k } {x \sim Q_\nu }\\
&= \sup_{j \in [d]} \expectsub{ \absv{ -x_j + \mu_{\nu,j} }^k } {x \sim Q_\nu} \\
& \le p \cdot \Paren{p^{-1/k}}^k = 1.
\end{align*}
}

Now we are able to bound the error of SCO.
\begin{align}
\expect{ L_{\mathcal{D}}(w^{priv}) - \underset{\hat{w} \in \mathcal{W}}{\min} L_{\mathcal{D}}(\hat{w})}  & \ge \frac1{\absv{\cV}}  \sum_{\nu \in \cV} \expect{ L_{Q_\nu}(w^{priv}) - \underset{\hat{w} \in \mathcal{W}}{\min} L_{Q_\nu}(\hat{w})  }\label{equ:low:avg}\\
& \ge \frac1{\absv{\cV}}  \sum_{\nu \in \cV}  \expect { \left\langle \frac{\mu_\nu}{\| \mu \|_2}, \mu_\nu \right\rangle - \left\langle  w^{priv}, \mu_\nu \right\rangle }\label{equ:low:two}\\
& = \frac1{\absv{\cV}}  \sum_{\nu \in \cV}  \expect {  \| \mu \|_2 -  \left\langle  w^{priv}, \mu_\nu \right\rangle}\nonumber\\
& \ge \frac1{\absv{\cV}}  \sum_{\nu \in \cV}  \expect { \frac{1}{2} \cdot \| \mu \|_2 \cdot \| w^{priv} - w_\nu  \|^2_2}\label{equ:low:three},
\end{align}
where~\eqref{equ:low:avg} comes from the fact that the worst case loss is no smaller than the average loss,~\eqref{equ:low:two} comes from  $ w_\nu = \underset{\hat{w}\in \mathcal{W}}{\text{argmin}}  L_{Q_\nu}(\hat{w})$, and~\eqref{equ:low:three}  comes from the fact that $\norm{w^{priv}} \le 1$, and  $\norm{w_\nu}\le 1$.

%

Note that $\absv{\cV} \ge 2^{\frac{d}{8}}$. Furthermore, $\forall \nu \neq \nu^\prime$, $\norm{w_\nu - w_{\nu^\prime}} = \Omega\Paren{1}$; $\dtv\Paren{w_\nu, w_{\nu^\prime}} = p$, indicating that there exists a coupling between $w_\nu$ and $w_{\nu^\prime}$ with a coupling distance $np$. Suppose $p = \min\Paren{1, \frac{d}{n\epsilon}}$, by DP Fano's inequality (Theorem 2 in~ \cite{AcharyaSZ21}), it can be shown that
\[
\frac1{\absv{\cV}}  \sum_{\nu \in \cV}  \expect { \| w^{priv} - w_\nu  \|^2_2} = \Omega\Paren{1}. 
\]

Thus,
\begin{align}
\label{equ:lb_c_m}
\expect{L_{\mathcal{D}}(w^{priv}) - \underset{\hat{w} \in \mathcal{W}}{\min} L_{\mathcal{D}}(\hat{w})} \geq  \Omega\Paren{1} \cdot \norm{\mu} =  {\Omega}\Paren{ \min\Paren{1, \left( \frac{d}{\epsilon n}\right)^{\frac{k-1}{k}}} \newhz{\cdot \sqrt{d}} }.
\end{align}

\newhz{
With respect to $\rho$-CDP algorithms, we just take $p = \min \Paren{1, \frac{\sqrt{d}}{\ns\sqrt{\rho}}}$, and replace DP Fano's inequality by CDP Fano's inequality, then all the proof follows.
}

Now we prove the first term. We generally follow the lower bound proof of estimating Gaussians~\cite{AcharyaSZ21}.  Given $\nu \in \{0,1\}^d$, we define $Q_{\nu} = \normal(\mu_\nu, \mathbb{I}_d)$, where $\mu_\nu = \frac{p}{\sqrt{d}}\cdot \nu$, for some $p\in[0,1]$. Similarly, we define $w_\nu = \frac{\mu_\nu}{\norm{\mu_\nu}}$.

As a standard Gilbert-Varshamov bound for constant-weight codes (e.g., see Lemma 6 in~\cite{AcharyaSZ21}), there exists a set $\cV$ with cardinality at least $\norm{\cV} \ge 2^{\frac{d}{8}}$, with $\normo{\nu}=\frac{d}{2}$ for all $\nu \in \cV$, and with $\dham\Paren{\nu, \nu^\prime} \ge \frac{d}{2}$ for all $\nu \neq \nu^\prime \in \cV$.

Suppose $p = \min\Paren{1, \sqrt{\frac{d}{n}}}$, we can compute the norm of the distribution mean. Note that $\normo{\nu}=\frac{d}{2}$,
$$
\| \mu_\nu \|_2  = \frac{\sqrt{2}}{2}  \min\Paren{1, \sqrt{\frac{d}{n}}} \coloneqq \norm{\mu}.
$$

By a similar argument with the private case, it can be shown that
\begin{align*}
\expect{ L_{\mathcal{D}}(w^{priv}) - \underset{\hat{w} \in \mathcal{W}}{\min} L_{\mathcal{D}}(\hat{w})} \ge \frac{\norm{\mu}}{8} \cdot \frac1{\absv{\cV}}  \cdot \sum_{\nu \in \cV}  \expect {\| \hat{w}^{priv} - w_\nu  \|^2_2},
\end{align*}
where $\hat{w}^{priv} \coloneqq \arg\min_{\nu \in \cV} \norm{w_\nu - w^{priv}}$.

Note that this is indeed a multi-way classification problem, where $w_\nu$'s are well-separated. By classical Fano's inequality, 

\[
\frac1{\absv{\cV}}  \sum_{\nu \in \cV}  \expect { \| \hat{w}^{priv} - w_\nu  \|^2_2} = \Omega\Paren{1}. 
\]

Thus,
\begin{align}
\label{equ:lb_c_m2}
\expect{L_{\mathcal{D}}(w^{priv}) - \underset{\hat{w} \in \mathcal{W}}{\min} L_{\mathcal{D}}(\hat{w})} \geq  \Omega\Paren{1} \cdot \norm{\mu} =  {\Omega}\Paren{\min\Paren{1, \sqrt{\frac{d}{n}}}}.
\end{align}

Combining~\eqref{equ:lb_c_m} and~\eqref{equ:lb_c_m2}, and note that  $\forall a, b_1, b_2$, $a \ge 0.5(b_1+b_2)$ if $a\ge\max(b_1, b_2)$, we conclude the proof.



\subsection{Proof of Theorem~\ref{thm:dp_fano}}
\label{sec:dp_fano}

We note that the first term comes from classical Fano's inequality. So it is enough to prove the second term.

Let $\optim$ be a random variable uniformly sampled over $[M]$. Given $\optim$, we generate $\ns$ i.i.d.\ samples $X \sim p_{\optim}$. Note that the distribution of $X$ is a mixture of $M$ distributions. Specifically, for any event $S$, $$\probof{X \in S} = \frac1{M}\sum_{i \in M} \proboff{X \in S}{X \sim p_i^{\ns}}.$$
	
 Letting $X^i \sim p_i^\ns$, and $\hat{p}(X)$ be a classifier mapping from samples to the underlying distribution. For the mutual information $\mutual \Paren{X, \hat{p}(X)}$ between $X$ and $\hat{p}(X)$,
	\begin{align*}
	\mutual \Paren{X, \hat{p}(X)} &=  \expectsub{\dkl\Paren{\hat{p}(x),\hat{p}(X)}}{x \sim X}\\
	& = \frac1{M} \sum_{i \in M}  \expectsub{\dkl\Paren{\hat{p}(x),\hat{p}(X)}}{x \sim p_i^\ns},
	\end{align*}
	where the first equation comes from the definition of the mutual information: 
	\[
	I(X,Y) =  \expectsub{\dkl\Paren{{Y|X}, {Y}} }{X}.
	\]
	
	By convexity of the KL divergence,
	\begin{align*}
	\dkl\Paren{\hat{p}(x),\hat{p}(X)}&\le \frac1{M} \sum_{j \in M} \dkl\Paren{\hat{p}(x),\hat{p}(X^j)}\\
	&\le \frac1{M} \sum_{j \in M}  \expectsub{\dkl\Paren{\hat{p}(x),\hat{p}(x^\prime)}}{x^\prime \sim p_j^\ns}.
	\end{align*}
	
	Therefore,
	\begin{align*}
	\mutual \Paren{X, \hat{p}(X)} &\le \frac1{M^2} \sum_{i \in M} \sum_{j \in M}  \expectsub{\expectsub{\dkl\Paren{\hat{p}(x),\hat{p}(x^\prime)} }{x^\prime \sim p_j^\ns}}{x \sim p_i^\ns}.
	\end{align*}
	
	By the group privacy property of CDP (Proposition 1.9 in~\cite{BunS16}), which says that if $\hat{p}$ is $\rho$-CDP, $\dkl\Paren{\hat{p}(x),\hat{p}(x^\prime)} \le \rho \cdot  \dham\Paren{x,x^\prime}^2$. Therefore, we have 
	\begin{align*}
	\mutual \Paren{X, \hat{p}(X)} &\le \frac{\rho}{M^2} \sum_{i \in M} \sum_{j \in M}  \expectsub{\expectsub{ \dham\Paren{x,x^\prime}^2 }{x^\prime \sim p_j^\ns}}{x \sim p_i^\ns}.
	\end{align*}
	
  Note that the TV distance between each pair of distributions is upper bounded by $\alpha$. By the property of optimal coupling, there exists a coupling such that $\proboff{z \neq z^\prime} {z \sim p_i, z^\prime \sim p_j} = \alpha$. Therefore, $\dham\Paren{x,x^\prime} \sim \bin(n, \alpha) $, and
	\begin{align*}
	\expectsub{\expectsub{ \dham\Paren{x,x^\prime}^2 }{x^\prime \leftarrow p_j^n}}{x \leftarrow p_i^n} \le n^2\alpha^2 + n\alpha(1-\alpha).
	\end{align*}
	
	By Fano's inequality, let $p_e = \frac1{M} \sum_{i \in [M]} \probofsub{X \sim \p_i^\ns}{\hat{\p} (X) \neq \p_i}$,
	\[
	   \mutual \Paren{i^*, \hat{p}(X)} \ge (1- p_{e} ) \log M - \log 2.
	\]
	Noting that $\mutual \Paren{i^*, \hat{p}(X)} \le 	\mutual \Paren{X, \hat{p}(X)}$, combining inequalities shows that
	\begin{align}
	\label{equ:error_prob_lb}
	p_e  \ge 1 - \frac{\rho \Paren{n^2\alpha^2 + n\alpha(1-\alpha)} + \log2}{\log M}. 
	\end{align}
Finally, let $\hat{p}(X) \coloneqq \arg\min_{i \in M} \ell\Paren{\hat{\theta}(X), \theta(p_i)}$. By triangle inequality,
\[
\ell\Paren{ \theta(p_{i^*}), \theta(\hat{p}(X))} \le \ell\Paren{\hat{\theta}(X), \theta\Paren{ p_{i^*})}} + \ell\Paren{ \hat{\theta}(X), \theta(\hat{p}(X))} \le 2 \ell\Paren{\hat{\theta}(X), \theta\Paren{p_{i^*}}}.
\]

Therefore,
\begin{align*}
\frac1{M}  \sum_{i \in [M]}  \expectsub { \ell\Paren{\hat{\theta}(X) , \theta(\p_i)}}{X \sim  p_i^{\ns}} &\ge \frac1{2M}  \sum_{i \in [M]}  \expectsub { \ell\Paren{\theta(\hat{p}(X)) , \theta(\p_i)}}{X \sim  p_i^{\ns}} \\
&\ge \frac{r}{2M} \sum_{i \in [M]}   \probofsub{X \sim  p_i^{\ns}}{\hat{p}(X) \neq p_i} \\
&= \frac{r p_e}{2}.
\end{align*}
Combined with~\eqref{equ:error_prob_lb}, we conclude the proof.

\subsection{\textcolor{black}{Theorem~\ref{cor:corollary15} with High-probability Guarantees}}
\label{sec:proof_hp}
\allowdisplaybreaks

In this paper, we provide all our utility guarantees in terms of the expectation over the randomness of samples and algorithms. However, they can be easily generalized to the high-probability setting.
In this section, we present the high-probability version of Theorem~\ref{cor:corollary15} as an example.

\begin{theorem}[Theorem \ref{cor:corollary15} in high probability]\htodo{assumption on smoothness, $R$, and $n,d$}
\label{hp_bound}
  Suppose we have a stochastic convex optimization problem which satisfies Assumption~\ref{assn:main-assumption}. Assuming $R \le 10$, $\sm \le 10$,
 Algorithm~\ref{alg:a-dpscoht}, instantiated with CDPCWME with parameters \hz{$T = \frac{R^2\rho n^2}{\tau^2 d^4}$}, $\eta = \frac{M}{R \sqrt{T}}$, and $\tau = \Paren{\frac{\sqrt{\rho} n}{Md^{\frac{3}{2}}}}^{\frac1k}$, outputs \hz{$\priv =  \frac1T \sum_{t \in [T]} w^t$}, such that with probability at least $1-\beta$,
\begin{align*}
L_{\mathcal{D}}(\priv) - L_{\mathcal{D}}(w^*)   \leq O\Bigg( \frac{Md \sqrt{\log\frac{Mdn}{\beta}}}{\sqrt{n}}+ \log\left( \frac{Mdn}{\beta} \right) \cdot \Paren{\frac{Md^2}{n\sqrt{\rho}}\cdot \Paren{\frac{\sqrt{\rho}n}{Md^{\frac{3}{2} }}}^{\frac1{k}} + \frac{MLd^{1.5}}{n^3} }\Bigg),
\end{align*}
  where $w^* = \arg\min_w L_{\mathcal{D}}(w)$, and $M$ is the diameter of the constraint set $\cW$.
  \end{theorem}

\begin{proof}

Let ${w'}^t = w^{t-1}  - \eta \nabla \widetilde{L}_{\mathcal{D}}(w^{t-1})$, and $w^t$ denotes its projection to $\mathcal{W}$. Similar with the proof of Theorem~\ref{thm:sco-f-c},
\begin{align}
&L_{\mathcal{D}}(w^{priv}) - L_{\mathcal{D}}(w^*) \nonumber\\
\leq  & \frac{1}{T}\sum_{t=1}^T \frac{1}{\eta} \left\langle \eta  \nabla L_{\mathcal{D}}(w^t), w^t- w^* \right\rangle \nonumber\\
\leq  &\frac{1}{T}\sum_{t=1}^T \frac{1}{\eta} \left\langle \eta  \nabla L_{\mathcal{D}}(w^t) +  \eta \nabla \widetilde{L}_{\mathcal{D}}(w^t) -  \eta \nabla \widetilde{L}_{\mathcal{D}}(w^t), w^t- w^* \right\rangle \nonumber\\
\xtl{=}  & \frac{1}{T}\sum_{t=1}^T \left\langle \nabla L_{\mathcal{D}}(w^t) -   \nabla \widetilde{L}_{\mathcal{D}}(w^t), w^t- w^* \right\rangle  + \frac{1}{T}\sum_{t=1}^T \frac{1}{\eta} \left\langle  \eta \nabla \widetilde{L}_{\mathcal{D}}(w^t), w^t- w^* \right\rangle \nonumber\\
= &\text{LHS} + \text{RHS}.  \nonumber
\end{align}

We first bound the RHS, which corresponds to analyzing the variance. Let ${w'}^t = w^{t-1}  - \eta \nabla \widetilde{L}_{\mathcal{D}}(w^{t-1})$, and $w^t$ denotes its projection to $\mathcal{W}$. Similar with the proof of Theorem~\ref{thm:sco-f-c},
\begin{align}
\text{RHS}=&\frac{1}{T}\sum_{t=1}^T \frac{1}{\eta} \left\langle  \eta \nabla \widetilde{L}_{\mathcal{D}}(w^t), w^t- w^* \right\rangle \nonumber \\
 = &\frac{1}{T}\sum_{t=1}^T     \Biggl(   \frac{1}{2\eta}    \biggl(   - \left\| w^t - w^* -\eta \nabla \widetilde{L}_{\mathcal{D}}(w^t)  \right\|^2 + \left\| w^t - w^* \right\|^2    \biggr)   + \frac{\eta}{2} \left\| \nabla  \widetilde{L}_{\mathcal{D}}(w^t) \right\|^2    \Biggr)  \nonumber \\
 = &\frac{1}{T}\sum_{t=1}^T\left( \frac{1}{2\eta} \left(-   \left\| {w'}^{t+1} - w^* \right\|^2   +  \left\| {w}^t - w^* \right\|^2  \right) +  \frac{\eta}{2} \cdot \left\| \nabla  \widetilde{L}_{\mathcal{D}}(w^t) \right\|^2 \right) \nonumber \\
 \leq & \frac{1}{T}\sum_{t=1}^T\left( \frac{1}{2\eta} \left(-   \left\| w^{t+1} - w^* \right\|^2   +  \left\| {w}^t - w^* \right\|^2  \right) +  \frac{\eta}{2} \cdot  \left\| \nabla  \widetilde{L}_{\mathcal{D}}(w^t) \right\|^2 \right) \nonumber\\
 \xtl{=} & \frac{1}{2\eta T}\left(-   \left\| w^T - w^* \right\|^2 +   \left\| w^1 - w^* \right\|^2  \right)  + \frac{\eta}{2T} \cdot  \sum_{t=1}^T \left\| \nabla  \widetilde{L}_{\mathcal{D}}(w^t) \right\|^2  \nonumber \\
  \xtl{=} & \frac{1}{2\eta T}\left(-   \left\| w^T - w^* \right\|^2 +   \left\| w^1 - w^* \right\|^2  \right)  + \frac{\eta}{2T} \cdot  \sum_{t=1}^T\left\| \nabla  \widetilde{L}_{\mathcal{D}}(w^t) - \nabla {L}_{\mathcal{D}}(w^t)+ \nabla {L}_{\mathcal{D}}(w^t) \right\|^2\nonumber \\
\leq & \frac{1}{2\eta T}\left(-   \left\| w^T - w^* \right\|^2 +   \left\| w^1 - w^* \right\|^2  \right)  + \frac{\eta}{T} \cdot  \sum_{t=1}^T\left({\left\| \nabla  \widetilde{L}_{\mathcal{D}}(w^t) - \nabla {L}_{\mathcal{D}}(w^t) \right\|^2+\left\| \nabla {L}_{\mathcal{D}}(w^t) \right\|^2} \right).\nonumber
\end{align}

By Assumption~\ref{assn:main-assumption}, we have 
\begin{align*}
\| \nabla  L_{\mathcal{D}}(w^t) \|^2 = \left\| \nabla  \underset{x \sim \mathcal{D}}{\mathbb{E}}[  \ell(w^t, x)] \right\|^2  = \left\|  \underset{x \sim \mathcal{D}}{\mathbb{E}}[  \nabla  \ell(w^t, x)] \right\|^2 \leq R^2
\end{align*}
for all $t$, and $\left\| w^{\prime} - w^* \right\|^2 \leq M^2$ for any $w^{\prime} \in \mathcal{W}$. 

Thus,

\begin{align}
\text{RHS} & \leq \frac{M^2}{2\eta T} + \eta R^2 + \frac{\eta}{T} \sum_{t=1}^{T} \left( \left\| \nabla  \widetilde{L}_{\mathcal{D}}(w^t) - \nabla {L}_{\mathcal{D}}(w^t) \right\|^2 \right). \nonumber\\
 &\leq \frac{M^2}{2\eta T} + \eta R^2 + \frac{2\eta}{T} \sum_{t=1}^{T} \left( \left\| \nabla  \widetilde{L}_{\mathcal{D}}(w^t) - \nabla \hat {L}_{\mathcal{D}}(w^t) \right\|^2 + \left\| \nabla  \hat{L}_{\mathcal{D}}(w^t) - \nabla {L}_{\mathcal{D}}(w^t) \right\|^2 \right). \nonumber
\end{align}

Following a similar proof with the covering argument in the proof of Theorem~\ref{cor:corollary15}, we can bound $\left\| \nabla  \hat{L}_{\mathcal{D}}(w^t) - \nabla {L}_{\mathcal{D}}(w^t) \right\|^2$ for all $t \in [T]$  simulatenously. Specifically, replacing $\alpha = \frac1{n^3}$ and $m = d \log\frac{dMn}{\beta}$, we can show that with probability $1-\frac{\beta}{10}$, for all $w \in \cW$,
\begin{align}
\label{equ:hp_covering}
\left\| \nabla  \hat{L}_{\mathcal{D}}(w) - \nabla {L}_{\mathcal{D}}(w) \right\| \le O\Paren{\sqrt{d}\cdot\Paren{\frac{C}{\tau}}^{k-1} + d \cdot \sqrt{\frac{\log\frac{Md n}{\beta} }{n}} + \frac{Ld^{1.5}\log\frac{Mdn}{\beta} }{n^3}}.
\end{align}
Then by Gaussian tail bound and the union bound, with probability $1-\frac{\beta}{10}$, for all $w_t$ with $t \in [T]$,
\[
\left\| \nabla  \hat{L}_{\mathcal{D}}(w) - \nabla \widetilde{L}_{\mathcal{D}}(w) \right\| \le O\Paren{\frac{\tau d^{2} \sqrt{T \log\frac{MndT}{\beta}}}{n\sqrt{\rho}}},
\]

Combining the previous two equations, with probability at least $1-\frac{\beta}{5}$,
\begin{align}
\label{equ:hp_gaussian}
&\frac{\eta}{T} \sum_{t=1}^{T} \left\| \nabla {L}_{\mathcal{D}}(w^t) - \nabla \hat {L}_{\mathcal{D}}(w^t) \right\|^2 \le O\Paren{ {\eta}\cdot \Paren {{d}\cdot\Paren{\frac{C}{\tau}}^{2k-2} +  {\frac{d^2\log\frac{Md n}{\beta} }{n}} + \frac{L^2d^{3}\log^2\frac{Mdn}{\beta} }{n^6}}}, \nonumber\\
&\frac{\eta}{T} \sum_{t=1}^{T} \left\| \nabla  \widetilde{L}_{\mathcal{D}}(w^t) - \nabla \hat {L}_{\mathcal{D}}(w^t) \right\|^2 \le O\Paren{{\eta} \cdot \frac{\tau^2 d^{4}{T \cdot \log\frac{MndT}{\beta}}}{n^2{\rho}}}.
\end{align}

Next we bound the LHS, which corresponds to analyzing the bias. By the triangle inequality,
\begin{align}
\nonumber
\text{LHS} &= \frac{1}{T}\sum_{t=1}^T \left\langle \nabla L_{\mathcal{D}}(w^t) -   \nabla \widetilde{L}_{\mathcal{D}}(w^t), w^t- w^* \right\rangle \\
&\leq \frac{1}{T}\sum_{t=1}^T \left\langle \nabla L_{\mathcal{D}}(w^t) -   \nabla \hat{L}_{\mathcal{D}}(w^t), w^t- w^* \right\rangle + \frac{1}{T}\sum_{t=1}^T \left\langle \nabla \hat{L}_{\mathcal{D}}(w^t) -   \nabla \widetilde{L}_{\mathcal{D}}(w^t), w^t- w^* \right\rangle. \nonumber
\end{align}

By~\eqref{equ:hp_covering}, we have
\begin{align*}
\left\| \nabla  \hat{L}_{\mathcal{D}}(w) - \nabla {L}_{\mathcal{D}}(w) \right\| \le O\Paren{\sqrt{d}\cdot\Paren{\frac{C}{\tau}}^{k-1} + d \cdot \sqrt{\frac{\log\frac{Md n}{\beta} }{n}} + \frac{Ld^{1.5}\log\frac{Mdn}{\beta} }{n^3}}.
\end{align*}
Note that $\norm{w^t-w^*}\le M$,
\begin{align*}
&\frac{1}{T}\sum_{t=1}^T \left\langle \nabla L_{\mathcal{D}}(w^t) -   \nabla \hat{L}_{\mathcal{D}}(w^t), w^t- w^* \right\rangle \leq M \cdot O\Paren{\sqrt{d}\cdot\Paren{\frac{C}{\tau}}^{k-1} + d \cdot \sqrt{\frac{\log\frac{Md n}{\beta} }{n}} + \frac{Ld^{1.5}\log\frac{Mdn}{\beta} }{n^3}}.
\end{align*}

Until now, the proof is almost the same with the case under expectation. Lastly, we analyze the term of $\nabla \hat{L}_{\mathcal{D}}(w^t) -   \nabla \widetilde{L}_{\mathcal{D}}(w^t)$. Note that this term is new, since it is zero when taking expectation. As designed in Theorem \ref{thm:main-variation}, we notice that 
\begin{align}
\nabla \hat{L}_{\mathcal{D}}(w^t) -   \nabla \widetilde{L}_{\mathcal{D}}(w^t) = N_t \sim \normal\Paren{0, \sigma \mathbb{I}_{d \times d}}, \nonumber
\end{align}
where $\sigma^2 = \frac{{72}\tau^2 d^3T \log^2\frac{Mnd}{\beta}}{\rho n^2}$. Note that $N_t$ is independent of $w^t - w^*$, with $\mathbb{E}[\langle N_t, w^t - w^* \rangle | w^t]=0 $. Therefore, $\{ \langle N_t, w^t - w^* \rangle, w^t \}^T_{t=0}$ is a martingale difference sequence. By Azuma's inequality for sub-Gaussian distributions (see Theorem 2 in \cite{Azuma_subg}),
\begin{align}
\frac{1}{T}\sum_{t=1}^T \left\langle \nabla \hat{L}_{\mathcal{D}}(w^t) -   \nabla \widetilde{L}_{\mathcal{D}}(w^t), w^t- w^* \right\rangle \leq O\left( \frac{M\sqrt{d}\sigma \log(1/\beta)}{\sqrt{T}} \right) \label{equ:hp3}
\end{align}
with probability at least $1-\frac{\beta}{10}$. Taking $\eta = \frac{M}{R\sqrt{T}}$, and $T=\frac{R^2\rho n^2}{\tau^2d^4}$, this term is strictly dominated by~\eqref{equ:hp_gaussian}. 

Finally, putting everything together with the same $\eta$ and $T$ chosen in Theorem \ref{cor:corollary15}, and by the union bound, we conclude the result in Theorem~\ref{hp_bound}.

\end{proof}

\end{document}